\documentclass{article}

\PassOptionsToPackage{numbers, compress}{natbib}

\usepackage[preprint]{neurips_2025}

\usepackage[utf8]{inputenc} 
\usepackage[T1]{fontenc}    
\usepackage[colorlinks=true]{hyperref}       
\usepackage{url}            
\usepackage{booktabs}       
\usepackage{amsfonts}       
\usepackage{nicefrac}       
\usepackage{microtype}      
\usepackage{xcolor}         

\usepackage{amsmath,amssymb,amsthm}

\renewcommand{\*}[1]{\mathcal{#1}}
\def\R{\mathbb{R}}
\def\E{\mathbb{E}}
\def\Eb{\mathbf{E}}
\def\X{\mathbf{X}}
\def\Y{\mathbf{Y}}
\def\A{\mathbf{A}}
\def\P{\mathbb{P}}
\def\SS{\boldsymbol{\Sigma}}
\def\MMu{\boldsymbol{\mu}}

\DeclareMathOperator{\CosSim}{\mathsf{ CosSim}}

\usepackage{mathtools}
\newtheorem{theorem}{Theorem}
\newtheorem{lem}{\protect\lemmaname}
\providecommand{\lemmaname}{Lemma}
\newtheorem{prop}{Proposition}

\newtheorem{assumption}{\protect\assumptionname\ignorespaces}
\providecommand{\assumptionname}{A}

\usepackage{enumitem}
\usepackage{multibib}
\newcites{app}{Appendix References}

\title{How high is `high'? Rethinking the roles of dimensionality in topological data analysis and manifold learning}

\author{%
  Hannah Sansford \\
  School of Mathematics\\
  University of Bristol, UK\\
   \And
     Nick Whiteley \\
  School of Mathematics\\
  University of Bristol, UK\\
     \And
     Patrick Rubin-Delanchy \\
  School of Mathematics\\
  University of Edinburgh, UK\\
}

\begin{document}
\maketitle


\begin{abstract}
We present a generalised Hanson-Wright inequality and use it to establish new statistical insights into the geometry of data point-clouds. In the setting of a general random function model of data, we clarify the roles played by three notions of dimensionality: \emph{ambient intrinsic dimension} $p_{\mathrm{int}}$, which measures total variability across orthogonal feature directions; \emph{correlation rank}, which measures functional complexity across samples; and \emph{latent intrinsic dimension}, which is the dimension of manifold structure hidden in data. Our analysis shows that in order for persistence diagrams to reveal latent homology and for manifold structure to emerge it is sufficient that $p_{\mathrm{int}}\gg \log n$, where $n$ is the sample size.  Informed by these theoretical perspectives, we revisit the ground-breaking neuroscience discovery of toroidal structure in grid-cell activity made by \citet{gardner2022toroidal}: our findings reveal, for the first time, evidence that this structure is in fact \emph{isometric} to physical space, meaning that grid cell activity conveys a geometrically faithful representation of the real world. 
\end{abstract}

\section{Introduction}\vspace{-0.1cm}
\paragraph{Data geometry in Machine Learning}
Consider a data set of $n$ samples and $p$ features, $\Y_1,\ldots,\Y_n\in\R^p$. A wide range of algorithms, models, supervised and unsupervised learning methods process such data by taking as inputs the pairwise Euclidean distances, dot-products or cosine similarities,
\vspace{-5pt}
\begin{align}\label{eq:pairwise}
\|\Y_i- \Y_j\|^2 = \|\Y_i\|^2 +  \|\Y_j\|^2 - 2 \Y_i \cdot \Y_j,\qquad \CosSim(\Y_i,\Y_j)=\frac{\Y_i \cdot \Y_j}{\|\Y_i\|\|\Y_j\|}.
\end{align}
The collection of these pairwise quantities across $i,j=1,\ldots,n$ conveys the geometric shape of $\mathcal{Y}_n\coloneqq\{\Y_1,\ldots,\Y_n\}$ regarded as a point-cloud in $\mathbb{R}^p$, and is the input to most dimension reduction and manifold learning methods, ranging from Classical Multidimensional Scaling \cite{torgerson1952multidimensional}, to Kernel PCA \cite{scholkopf1997kernel}, Isomap \cite{tenenbaum2000global}, $t$-SNE \cite{van2008visualizing} and UMAP \cite{mcinnes2018umap}, which are hugely popular for data visualisation, as well as spectral clustering \cite{von2007tutorial} (via a kernel function) and hierarchical, agglomerative clustering methods such as  UPGMA \cite{upgma,gray2023hierarchical}. The set of all pairwise distances is also the input to Persistent Homology techniques in topological data analysis (TDA) \cite{edelsbrunner2002topological,zomorodian2004computing,chazal2009proximity,chazal2021introduction}. Pairwise distances or dot-products are the input to classical supervised learning methods such as kernel methods and nearest-neighbour methods \cite{scholkopf2002learning}, and dot-products between training and test feature vectors define predictions in linear regression. 
\vspace{-0.2cm}
\paragraph{The disconnect between reality and existing statistical theory of data geometry.}  In this work we present new insights into statistical properties of the quantities in \eqref{eq:pairwise}, allowing us to explain how they convey latent geometry of the data-generating mechanism. In doing so we address two mismatches between the prominent statistical theory of data geometry and the realities of data analysis. Firstly, the popular mathematical view of a mean-zero, identity-covariance random vector is that, in high dimensions, its distribution is concentrated near the surface of a hypersphere, and that i.i.d. copies of such a vector are close to orthogonal with high probability, see   \cite{hall2005geometric}, the textbook of \citet{hastie01}[Sec. 2.5], and \cite{cai2013distributions} for refined analysis of similar i.i.d. setups. This rather degenerate limiting geometry does not seem expressive enough to accurately represent data in practice; it is incompatible with the widely accepted \emph{Manifold Hypothesis}  \cite{bengio2013representation, fefferman2016testing}, which is the premise of manifold learning and nonlinear dimension reduction, asserting that nominally high-dimensional data actually lie on a low-dimensional set embedded in high-dimensional space.  Secondly, much existing theory, e.g., \cite{hall2005geometric,ahn2007high,shen2016statistics,aoshima2018survey} focuses on the regime $p\gg n$ with which the phrase ``high-dimension low sample size asymptotics'' (HDLSS), and more broadly ``high-dimensional data'', are typically associated. This existing theory may therefore seem irrelevant in the many practical situations where $p\leq n$; indeed $p\gg n$ could be violated by simply collecting more data.  \vspace{-0.3cm}
\paragraph{Contributions.}
Our main contributions are as follows:
\begin{itemize}
[topsep=0pt,itemsep=-1ex,partopsep=1ex,parsep=1ex,leftmargin=0.3cm]
\item In section \ref{sec:HS} we present a generalised Hanson-Wright (GHW) inequality which quantifies the concentration behaviour of dot-products between possibly dependent random vectors with sub-Gaussian entries, and identify our first notion of dimension: \emph{ambient intrinsic dimension} $p_{\mathrm{int}}$.
\item In section \ref{sec:data_model} we identify a second notion of dimension: \emph{correlation rank}, and use our GHW inequality to explain how latent topological structure emerges from a random function model of data (from \cite{whiteley2022statistical}) as $p_{\mathrm{int}}$ grows;  we show under sub-Gaussian  assumptions that $p \gg n$ is not necessary; rather 
$p_{\mathrm{int}}\gg\log n$ suffices. Figure \ref{fig:sim_Example} gives a preview of the phenomenon. 
\item In section \ref{sec:tda_consistency} we consider consistency of TDA persistence diagrams in the regime $p_{\mathrm{int}}/\log n \to\infty$.
\item In section \ref{sec:isometry} we identify a third notion of dimension: \emph{latent intrinsic dimension}, and discuss how isometry between observed and latent geometry can be measured in practice.
\item In section \ref{sec:neuro} we revisit the neuroscience study of \cite{gardner2022toroidal}, and informed by our theory, extend their groundbreaking discoveries by  assessing evidence of isometry between grid cell activity and physical space.
\end{itemize}
\begin{figure}
\includegraphics[width=\textwidth]{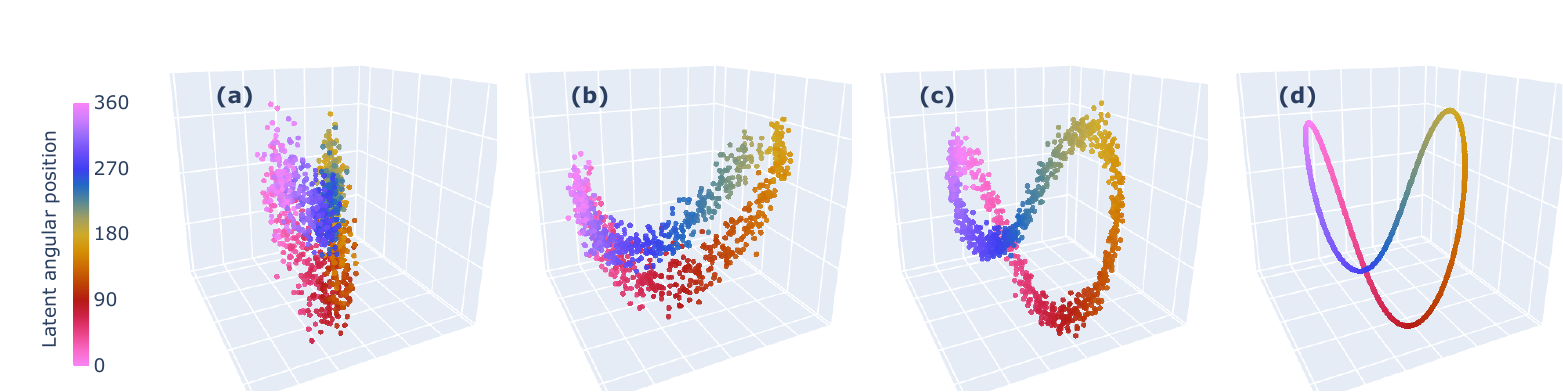}
\caption{\label{fig:sim_Example}Ambient intrinsic dimension $p_{\mathrm{int}}$,  correlation rank  $r$, and latent intrinsic dimension $d$ at play in simulation from a toy example of the random function model with $n=1000$. (a)-(c) show SVD visualisation of simulated data with respectively  $p_{\mathrm{int}}=3, 8, 20$; as $p_{\mathrm{int}}$ grows, the $d=1$-dimensional manifold $\mathcal{M}=\{\phi(z);z\in\mathcal{Z}\}$ shown in (d) emerges in a $r=3$-dimensional subspace.  In this example $\mathcal{M}$ is homeomorphic to the latent space $\mathcal{Z}$, which is a circle. See section \ref{sec:data_model} for details. }
\vspace{-15pt}
\end{figure}
Section \ref{sec:data_model} is directly influenced by geometric perspectives of \cite{whiteley2021matrix,whiteley2022statistical,gray2023hierarchical} and in particular the insightful statistical treatment of the manifold hypothesis \cite{whiteley2022statistical}, but our main theoretical results (theorems \ref{thm:HW_inequality}, \ref{thm:rand_func} and proposition \ref{prop:rand_func}) substantially extend and refine some of their results; we allow more general dependence structure across features and exploit sub-Gaussianity.   An extended discussion of related work is given in the supplementary material, along with all proofs and additional numerical experiments.\vspace{-0.3cm}
\paragraph{Notation and conventions.} We write $[n]\coloneq\{1,\ldots,n\}$ and $\ell_2$ for the set of square-summable real sequences $\{x=(x_1,x_2,\ldots):\|x\|\coloneqq(\sum_k|x_k|^2)^{1/2}<\infty\}$. For two nonnegative sequences $(a_n)_{n\geq 1}$, $(b_n)_{n\geq 1}$, $a_n\in\Omega(b_n)$ means $\liminf_{n\to\infty} a_n/b_n>0$. For two sequences of nonnegative random variables $(X_n)_{n\geq1}$, $(Y_n)_{n\geq 1}$, $X_n\in O_{\P}(Y_n)$ means that for any $\epsilon>0$ there exists $M$ and $n_0$ such that $\P(X_n/Y_n>M)\leq \epsilon$ for all $n\geq n_0$.    The sub-Gaussian norm of a random variable $X$ is $
\|X\|_{\psi_2} \coloneqq \sup_{q\geq 1} q^{-1/2}\E \left[|X|^q\right]^{1/p}$,
and $X$ is said to be sub-Gaussian if $\|X\|_{\psi_2}<\infty$. This condition can be understood as meaning that the tails of the distribution of $X$ decay at least as quickly as a Gaussian, but includes the case where $X$ is a discrete random variables taking only finitely many different values, or more generally the case where the support of the distribution of $X$ is bounded. The Frobenius and spectral matrix norms are respectively denoted $\|\cdot\|_{\mathrm{F}}$ and $\|\cdot\|$. $\mathbf{I}[\cdot]$ is the indicator function, $\mathbf{I}_p$ is the $p$-by-$p$ identity matrix.

\vspace{-0.3cm}
\section{A generalised Hanson-Wright inequality}\label{sec:HS}
The Hanson-Wright (HW) inequality is a concentration inequality for the quadratic form: $\X^\top \A \X$, where $\X$ is a vector of independent, sub-Gaussian random variables and $\A$ is a matrix \cite{hanson1971bound,wright1973bound,rudelson2013}, see supplementary material for more background. Our GHW inequality concerns $\X^\top \A \X^\prime$, where $\X,\X^\prime$ are allowed to be dependent, reducing to the standard HW inequality of \cite{rudelson2013} when $\X=\X^\prime$.  
\begin{theorem}
\label{thm:HW_inequality}Let $\X=(X_{1},\ldots,X_{p})$ and
$\X^{\prime}=(X_{1}^{\prime},\ldots,X_{p}^{\prime})$ be $\R^{p}$-valued
random vectors such that the pairs $(X_{j},X_{j}^{\prime})$, $j=1,\ldots,p$
are mutually independent, and $\E[X_{j}]=\E[X_{j}^{\prime}]=0$
and $\|X_{j}\|_{\psi_{2}}\vee\|X_{j}^{\prime}\|_{\psi_{2}}\leq K$
for all $j=1,\ldots,p$. Let $\A\in\mathbb{R}^{p\times p}$.
Then there is an absolute constant $c$ such that for any $t\geq0$, 
\[
\P\left(\left|\X^{\top}\A\X^{\prime}-\E\left[\X^{\top}\A\X^{\prime}\right]\right|>t\right)\leq2\exp\left[-c\min\left\{ \frac{t^{2}}{K^{4}\|\A\|_{\mathrm{F}}^{2}},\frac{t}{K^{2}\|\A\|}\right\} \right].
\]
\end{theorem}
\paragraph{Zero-mean i.i.d. vectors are close to orthogonal.}
As a first illustration of how theorem \ref{thm:HW_inequality} can be used to understand data geometry, we re-visit a conventional perspective on zero-mean, i.i.d. random vectors. Let $\X_1,\ldots,\X_n$ be i.i.d. copies of the random vector $\X$ in theorem \ref{thm:HW_inequality}  subject to the additional requirement that its elements satisfy $\E[|X_j|^2]=1$ for $j=1,\ldots,p$, let $\boldsymbol{\Sigma}\in\mathbb{R}^{p\times p}$ be positive semidefinite and define $\Y_i\coloneqq \boldsymbol{\Sigma}^{1/2}\X_i $. Then $\Y_1,\ldots,\Y_n$ are i.i.d., each have mean zero and covariance matrix $\SS$, and for $i\neq j$, $\E[\Y_i\cdot\Y_j]=0$. We define the \emph{ambient intrinsic dimension}: 
\begin{equation}    
p_\mathrm{int} \coloneqq \frac{\mathrm{tr}(\SS)}{ \|\SS\|}.
\end{equation}
We have $1\leq p_{\mathrm{int}} \leq \mathrm{rank}(\SS)\leq p$, with $p_{\mathrm{int}}=p$ when all $p$ eigenvalues of $\SS$ are equal.  Since $\mathrm{tr} (\SS)=\E[\|\Y_i-\E[\Y_i]\|^2]$, $p_{\mathrm{int}}$ can be understood as a normalised measure of total variability of $\Y_i$.
The quantity $\mathrm{tr}(\SS)/\|\SS\|$ was called \emph{intrinsic dimension} by \citet{tropp2015introduction} and \citet{vershynin2018high}, we include the prefix ``ambient'' to disambiguate between $p_\mathrm{int}$ and the notion of \emph{latent intrinsic dimension} defined in section \ref{sec:isometry}. 
\begin{prop}\label{prop:iid}
Let $\Y_1,\ldots,\Y_n$ be the i.i.d., zero-mean random vectors defined above. If $p_{\mathrm{int}}\in \Omega(\log n)$, then 
\begin{equation}\label{eq::iid}
\max_{i,j\in[n]}\left|\frac{ \mathbf{Y}_{i}\cdot\mathbf{Y}_{j}}{\mathrm{tr}(\SS)}-\mathbf{I}[i=j]\right|\in O_{\mathbb{P}}\left(\sqrt{\frac{\log n}{p_{\mathrm{int}}}}\right).
\end{equation}
\end{prop}
The proof of proposition  \ref{prop:iid} involves applying theorem \ref{thm:HW_inequality} with $\X=\X_i$, $\X^\prime = \X_j$, $\A=\SS$, so that $\X^\top \A \X=\Y_i\cdot\Y_j$, and taking a union bound over $i,j\in[n]$.
 Proposition \ref{prop:iid} tells us that if $p_{\mathrm{int}}\gg \log n$, then $\Y_1,\ldots,\Y_n$ tend to be uniformly close to orthogonal, and all have magnitudes uniformly close to $\mathrm{tr}(\SS)$. This geometric configuration is often referred to as ``high-dimensional'' behaviour of random vectors, but proposition \ref{prop:iid} shows that ``high-dimensional'' is perhaps a misnomer in this situation, since for the l.h.s. of \eqref{eq::iid} to be small with high probability, we only require $p_{\mathrm{int}}\gg \log n$. 
 

\vspace{-0.1cm}
\section{Topology and manifold structure in data emerges from  random functions }\label{sec:data_model}\vspace{-0.1cm}
\paragraph{Model definition and statistical properties.} In this section we define a general form of random function statistical model which is substantially more expressive than the i.i.d. setting of proposition \ref{prop:iid}. The setup of the model follows \cite{whiteley2022statistical} in part but not all; our model allows for more general dependence structure across features. Similary to \cite{whiteley2022statistical} though, it is not our aim to perform confirmatory statistical analysis or model fitting, rather the utility of the model is to help us understand performance of TDA and manifold learning.   The main ingredients of the model are:
\begin{itemize}
[topsep=0pt,itemsep=0ex,partopsep=1ex,parsep=1ex,leftmargin=0.3cm]
\item a metric space $(\*Z,d_{\*Z})$, where $\*Z$ is a set and $d_{\*Z}(\cdot,\cdot)$ is a distance function, and a  collection of points $z_1,\ldots,z_n\in\*Z$. We write $\mathcal{Z}_n\coloneqq\{z_1,\ldots,z_n\}$.
\item random, $\mathbb{R}$-valued functions, $X_1(\cdot),\ldots,X_p(\cdot)$, each with domain $\mathcal{Z}$, so for each $z\in\mathcal{Z}$ and $j\in[p]$, $X_j(z)$ is a random variable. We write in vector form $\X(z)\coloneqq(X_1(z),\ldots,X_p(z))$.
\item  positive semidefinite matrices $\{\SS(z)\in\R^{p\times p};z \in \*Z\}$ and vectors $\{\boldsymbol{\mu}(z)\in\mathbb{R}^p;z\in \*Z\}$;
\item  random `noise' vectors $\Eb_1,\ldots,\Eb_n\in\R^p$, which are independent of the $X_j$.
\end{itemize}
We then define, with $\sigma\geq 0$,
$$ \Y_i  \coloneqq \boldsymbol{\Sigma}^{1/2}(z_i)\X(z_i) + \boldsymbol{\mu}(z_i) +\sigma \Eb_i.
$$
We denote the `noise-free' component of the model $\Y^{\mathrm{nf}}(z)\coloneqq \boldsymbol{\Sigma}^{1/2}(z)\X(z) + \boldsymbol{\mu}(z) $, so that $\Y_i \equiv \Y^{\mathrm{nf}}(z_i)+\sigma\Eb_i$.  

\begin{assumption}\label{ass:X_j} The random functions $X_j$ are independent across $j$, and for every $j$ and $z$, $\E[X_j(z)]=0$, $\E[|X_j(z)|^2]=1$ and $\|X_j(z)\|_{\psi_2}\leq K$, for some $K<\infty$.
\end{assumption}
\begin{assumption}\label{ass:E}
The random vectors $\Eb_1,\ldots,\Eb_n$ are i.i.d., $\E[\Eb_i]=\mathbf{0}$, $ \E[\Eb_i\Eb_i^\top]=\mathbf{I}_p$, and the elements $E_i^{(j)}$, $j=1,\ldots,p$ of each $\Eb_i$ are independent and for all $j$, $\|E_i^{(j)}\|_{\psi_2}<K$ for some $K<\infty$. 
\end{assumption}
This model relaxes both the \emph{independence} and \emph{identical distribution} parts of the i.i.d. assumption in proposition \ref{prop:iid}: we allow that for any $j\in[p]$ and $z,z^\prime\in\*Z$, $X_j(z)$ and $X_j(z^\prime)$ may be dependent, hence $\Y_i$ and $\Y_j$ may be dependent; and \textbf{A\ref{ass:X_j}}  and \textbf{A\ref{ass:E}} imply that  $\Y_i$ is a random vector with mean $\boldsymbol{\mu}(z_i)$ and covariance matrix $\boldsymbol{\Sigma}(z_i)+\sigma^2 \mathbf{I}_p$. Theorem \ref{thm:rand_func} shows that despite this general dependence structure, dot products amongst vectors $\Y_1,\ldots,\Y_n$ are concentrated about their expectations in a manner analogous to proposition \ref{prop:iid}. Extending the notion of ambient intrinsic dimension to the present setting, we define:
$$p_{\mathrm{int}}^{(i)} \coloneqq \frac{\mathrm{tr}[\SS(z_i)]}{\|\SS(z_i)\|}.$$
\begin{theorem}\label{thm:rand_func} Assume that $\Y_1,\ldots,\Y_n$ follow the random function model, \textbf{A\ref{ass:X_j}}-\textbf{A\ref{ass:E}} hold, and $\min_{i\in[n]}p_{\mathrm{int}}^{(i)}\in\Omega(\log n)$ as $n\to\infty$. Then
$$
\max_{i,j\in[n]}\left|\frac{\Y_i\cdot \Y_j}{\E[\|\Y_i\|^2]^{1/2}\E[\|\Y_j\|^2]^{1/2}}-\frac{\E[\Y_i\cdot \Y_j]}{\E[\|\Y_i\|^2]^{1/2}\E[\|\Y_j\|^2]^{1/2}}\right| \in O_{\P}\left(\sqrt{\frac{\log n}{\min_{i\in[n]}p_{\mathrm{int}}^{(i)}}}\right), 
$$
as $n\to\infty$.
\end{theorem}
We shall next see that under additional but mild assumptions, the quantities $\E[\|\Y_i\|^2]$ and $\E[\Y_i \cdot \Y_j]$ 
have a rich geometric interpretation, which can be transferred to $\mathcal{Y}_n=\{\Y_1,\ldots,\Y_n\}$ by theorem \ref{thm:rand_func}.

\paragraph{Relationship between point-clouds $\mathcal{Y}_n$ and $\mathcal{M}_n$.} Consider the kernel function:
\begin{equation}\label{eq:kernel_defn}
(z,z^\prime)\mapsto \E[\Y^{\mathrm{nf}}(z)\cdot \Y^{\mathrm{nf}}(z^\prime)].
\end{equation}
and the assumptions:
\begin{assumption}\label{ass:MScontinuity}$z\mapsto \Y^{\mathrm{nf}}(z)$ is mean-square continuous, i.e., $z\to z^\prime$ implies $\E[\|\Y^{\mathrm{nf}}(z)-\Y^{\mathrm{nf}}(z^\prime)\|^2]\to 0$.
\end{assumption}
\begin{assumption}\label{ass:Z_compact}
The metric space $(\*Z,d_{\*Z})$ is compact.
\end{assumption}
Under \textbf{A\ref{ass:MScontinuity}} the kernel function in \eqref{eq:kernel_defn} is continuous, and if additionally \textbf{A\ref{ass:Z_compact}} holds,  for any finite Borel measure $\nu$ supported on $\*Z$, Mercer's theorem \citep{steinwart2008support}[Thm 4.49], tells us that the kernel function in \eqref{eq:kernel_defn} has the representation:
\vspace{-8pt}
\begin{equation}\label{eq:mercer}
\E[\Y^{\mathrm{nf}}(z)\cdot \Y^{\mathrm{nf}}(z^\prime)] = \phi(z)\cdot\phi(z^\prime)  = \sum_{k=1}^r \lambda_k u_k(z) u_k(z^\prime)
\end{equation}
where $\phi:\*Z\to \ell_2$ is conventionally called a \emph{feature map} and the $k$th element of the vector $\phi(z)$ is $\lambda_k^{1/2} u_k(z)$, where the $u_k$ are $L_2(\nu)$-orthonormal functions, the $\lambda_k$ are associated eigenvalues and $r\in\{1,2,\ldots\}\cup\{\infty\}$ is the largest $k$ such that $\lambda_k>0$. We call $r$ the  \emph{correlation rank}, which can be understood as a measure of functional complexity of the kernel function, and hence the correlation between $\Y^{\mathrm{nf}}(z)$ and $\Y^{\mathrm{nf}}(z^\prime)$:  $r$ counts the number of orthonormal functions needed to express the dot product $\phi(z)\cdot\phi(z^\prime)$ across all $z,z^\prime\mathcal{Z}$. Define:
$$
\mathcal{M} \coloneqq \{\phi(z);\,z\in\*Z\},\qquad \*M_n\coloneqq\{\phi(z_i);\,i\in [n]\}.
$$
It follows from \eqref{eq:mercer} and \textbf{A\ref{ass:E}} that $$
\E[\Y_i\cdot\Y_j] = \phi(z_i)\cdot \phi(z_j) +p\sigma^2 \mathbf{I}[i=j]. 
$$
Combined with theorem \ref{thm:rand_func}, this tells us that when $\min_i p_{\mathrm{int}}^{(i)}\gg n$,  the geometry of the point-cloud $\mathcal{Y}_n$ reflects that of $\*M_n$ up to some distortion depending on $\sigma$. The normalizing terms $\E[\|\Y_i\|^2]$ appearing in theorem \ref{thm:rand_func}  are typically unknown in practice; proposition \ref{prop:rand_func} explains the behaviour of $\Y_i\cdot\Y_j$ and $\CosSim(\Y_i,\Y_j)$ avoiding unknown normalization.
\begin{prop}\label{prop:rand_func} If \textbf{A\ref{ass:X_j}}-\textbf{A\ref{ass:Z_compact}} hold, then:
\begin{multline*}
\max_{i,j\in[n]}\left|\frac{\Y_i\cdot \Y_j}{p}-\frac{\phi(z_i)\cdot\phi(z_j)}{p}-\sigma^2\mathbf{I}[i=j]\right|\\ \in O_{\P}\left(\left[\max_{i\in[n]}\frac{\mathrm{tr}[\SS(z_i)]+\|\boldsymbol{\mu}(z_i)\|^2}{p}+\sigma^2\right]\sqrt{\frac{ \log n}{\min_{i\in[n]}p_{\mathrm{int}}^{(i)}}}\right) 
\end{multline*} 
 and with 
$$\gamma_{ij}(\sigma)\coloneqq\gamma_i(\sigma)\gamma_j(\sigma),\qquad \gamma_i(\sigma)
\coloneqq \frac{(\|\phi(z_i)\|^2+p\sigma^2)^{1/2}}{\|\phi(z_i)\|},
$$
$$
\max_{i\neq j\in[n]}\left|\CosSim(\Y_i,\Y_j)-\frac{\CosSim(\phi(z_i),\phi(z_j))}{\gamma_{ij}(\sigma)}\right| \in O_{\P}\left(\sqrt{\frac{\log n}{\min_{i\in[n]}p_{\mathrm{int}}^{(i)}}}\right). 
$$
\end{prop}
Recalling \eqref{eq:pairwise}, the first part of proposition \ref{prop:rand_func} implies that if $\max_{i\in[n]}/(\mathrm{tr}[\SS(z_i)]+\|\boldsymbol{\mu}(z_i)\|^2)/p+\sigma^2\in O(1)$  and $\min_i p_{\mathrm{int}}^{(i)}\gg n$, then for $i\neq j$,
\vspace{-5pt}
\begin{equation}\label{eq:distance}
\frac{1}{p}\|\Y_i-\Y_j\|^2\approx \frac{1}{p}\|\phi(z_i)-\phi(z_j)\|^2 +2\sigma^2.
\end{equation}
For the second part of proposition \ref{prop:rand_func}, notice that $\gamma_{ij}{(0)}=1$, and for any $\sigma>0$, if $\|\phi(z)\|$ is constant in $z$ (e.g. $\mathcal{Z}$ is a vector space and the kernel \eqref{eq:kernel_defn} is a function of $z-z^\prime$), then $\gamma_{ij}^{(\sigma)}$ is constant in $i,j$.

\paragraph{Relationship between metric space $\mathcal{Z}$ and manifold $\mathcal{M}$.}

Lemma \ref{lem:homeo} shows that under a mild non-degeneracy condition on $\Y^{\mathrm{nf}}(\cdot)$, $\mathcal{M}$ is topologically equivalent to $\*Z$.
\begin{assumption} \label{ass:homeo}
If $z\neq z^\prime$, then $\E[\|\Y^{\mathrm{nf}}(z)-\Y^{\mathrm{nf}}(z^\prime)\|^2]>0$. 
\end{assumption}
\begin{lem}\label{lem:homeo}Under \textbf{A\ref{ass:MScontinuity}}-\textbf{A\ref{ass:homeo}},  $\phi$ is a homeomorphism between $\*Z$ and  $\mathcal{M}$, i.e., $\phi$ is continuous, invertible on its image, and has a continuous inverse.
 \end{lem}
Informally, this means $\mathcal{M}$ resembles $\*Z$ subject to some transformation such as bending, twisting or stretching, but not cutting, or puncturing; $\mathcal{Z}$ and $\mathcal{M}$ must have the same number of connected components, the same number of one-dimensional holes, two-dimensional cavities, and so on. The term \emph{topological manifold} conventionally refers to some topological space which is \emph{locally} homeomorphic to a subset of Euclidean space; we can therefore speak of $\*M$ as a topological manifold, but one which is \emph{globally} homeomorphic to the metric space $\*Z$.  

In combination,  theorem \ref{prop:rand_func}, proposition \ref{thm:rand_func} and lemma \ref{lem:homeo} tell us that if the points in $\*Z_n$ are distributed across $\mathcal{Z}$, then $\mathcal{M}_n$ and hence $\mathcal{Y}_n$ will convey the `shape' of $\*Z$ -- in the remaining sections of the paper we will explore implications of this in TDA and manifold learning. 


\paragraph{A toy example: $\*Z$ is a circle. } Let $\*Z=\{z=(z_1,z_2):z_1^2+z_2^2=1\}$,  $r=3$, $\phi(z)=p^{1/2}[z_1,\frac{2}{\pi}\sin(\pi z_2/2),\frac{2}{\pi}\cos(\pi z_2/2)]$. The random functions $X_j$ are i.i.d., zero-mean Gaussian processes each with covariance function $p^{-1}\phi(z)\cdot\phi(z^\prime)$, $\boldsymbol{\mu}(z)=0$, $\SS(z)=\mathbf{I}_p$ and the elements of $\Eb_i$ are distributed $\mathcal{N}(0,1)$, and $\sigma^2=0.02$. Figure \ref{fig:sim_Example} shows the first $r=3$ dimensions of the SVD embedding of $\mathcal{Y}_n$ for $p_{\mathrm{int}}^{(i)}=p=3,8,20$, with $n=1000$, and $z_1,\ldots,z_{1000}$ uniformly spaced around $\*Z$. As $p$ increases the manifold $\mathcal{M}$  emerges, and is homeomorphic to $\mathcal{Z}$.

\section{Consistency of persistence diagrams}\label{sec:tda_consistency}
In this section we discuss how, under the model from section \ref{sec:data_model}, TDA can be applied to $\mathcal{Y}_n$ in order to estimate certain topological characteristics of $\mathcal{M}$ and hence $\mathcal{Z}$. In the mathematical framework of persistent homology, a topological space has associated \emph{Betti numbers} $H_0,H_1,H_2,\ldots$, respectively indicating the number of connected components, 1D holes, 2D cavities, etc.,  which the space exhibits. Two homeomorphic spaces, such as $\mathcal{Z}$ and $\mathcal{M}$ in the setting of lemma \ref{lem:homeo}, have the same Betti numbers. TDA techniques \cite{edelsbrunner2002topological,zomorodian2004computing,chazal2009proximity,chazal2021introduction}, for example implemented in the python package \verb|ripser| \cite{Tralie2018}, allow \emph{persistence diagrams} associated with data point-clouds to be computed (an example is shown in figure \ref{fig:torus}(c)), in turn enabling Betti numbers of the underlying space to be estimated.  

We view the sets $p^{-1/2}\*Y_n$, $p^{-1/2}\*M_n$ and $p^{-1/2}\*M$  as a metric spaces by equipping them with Euclidean (i.e., $\ell_2$) distance. We denote by $\text{dgm}(\cdot)$ persistence diagrams under some common choice of filtration\footnote{which could be either the Rips, \v{C}ech, or Alpha filtration.}. In the supplementary material we also discuss TDA operating on the normalised vectors $\Y_i/\|\Y_i\|$. 
A careful combination of existing results \citep{chazal2013optimal,ivanov2016realizations} detailed in the supplementary material gives
\vspace{-5pt}
\begin{multline}
d_{\text{b}}\left(\text{dgm}(p^{-1/2}\mathcal{Y}_n), \text{dgm}(p^{-1/2}\mathcal{M})\right) \\ \leq 2 (d_{\text{H}}(p^{-1/2}\mathcal{M}_n, p^{-1/2}\mathcal{M}) + d_{\text{GH}}(p^{-1/2}\mathcal{Y}_n,  p^{-1/2}\mathcal{M}_n)), \label{eq:bottleneck}
\end{multline}
where $d_{\text{b}}$ is the bottleneck distance, a distance between persistence diagrams; $d_{\text{GH}}$ is the Gromov-Hausdorff distance, a distance between metric spaces; and $d_{\text{H}}$ is the Hausdorff distance, a distance between subsets of a metric space \cite{burago2001course}. 

The persistence diagram $\text{dgm}(p^{-1/2}\mathcal{Y}_n)$ is said to be consistent if the left-hand-side of \eqref{eq:bottleneck} converges to zero in probability. The first term on the right-hand-side can be shown to vanish as $n \rightarrow \infty$, if $\phi(z_1), \ldots, \phi(z_n)$ are i.i.d. from a measure on $\mathcal{M}$ satisfying standard regularity conditions \citep{chazal2013optimal}. Proposition \ref{prop:rand_func} allows us to show that, up to an additive term depending only on $\sigma$, 
\emph{the second term vanishes}, as long as $\max_{i\in[n]}(\mathrm{tr}[\SS(z_i)]+\|\boldsymbol{\mu}(z_i)\|^2)/p+\sigma^2\in O(1)$ and  $\min_{i\in[n]}p_\mathrm{int}^{(i)}/\log n \to\infty$:
\[d_{\text{GH}}^2(p^{-1/2}\mathcal{Y}_n,  p^{-1/2}\mathcal{M}_n) \leq \max_{i,j \in [n]}\frac{1}{p}\left|\Y_i \cdot \Y_j - \phi(z_i) \cdot \phi(z_j) \right| \leq \sigma^2 +O_\mathbb{P}\left(\sqrt{\frac{\log n}{\min_{i\in[n]}p_{\mathrm{int}}^{(i)}}}\right)
,\]
where the first inequality is shown in the supplementary material. This tells us that accurate estimation of topology is possible without requiring $p\gg n$.
\section{Looking for evidence of isometry}\label{sec:isometry}

Isometry is a stronger form of relationship between between $\mathcal{Z}$
and $\mathcal{M}$ than homeomorphism, requiring that shortest path
lengths on $\mathcal{M}$ faithfully represent those on $\mathcal{Z}$ \cite{burago2001course}.
To define isometry mathematically, consider the random function model
with, for some $d\leq\tilde{d}$, $\mathcal{Z}$ a smooth, locally
$d$-dimensional subset of $\mathbb{R}^{\tilde{d}}$, with $d_{\mathcal{Z}}$
being Euclidean distance. Examples of such $\mathcal{Z}$ include,
in the case of $d=1$, a line segment, curve or circle (as in the
toy example above), and in the case $d=2$, a disk, sphere, or
torus. We refer to $d$ as \emph{latent intrinsic dimension.}

For any two points $z,z^{\prime}\in\mathcal{Z}$, a \emph{path} with end-points
$z,z^{\prime}$ is a smooth curve $\eta:[0,1]\to\mathcal{Z}$, such
that $\eta_{0}=z$ and $\eta_{1}=z^{\prime}$; similarly a path between
$x,x^{\prime}\in\mathcal{M}$ is a smooth curve $\gamma:[0,1]\to\mathcal{M}$,
with $\gamma_{0}=x$, $\gamma_{1}=x^{\prime}$. The lengths of these
paths are $L(\eta)\coloneqq\int_{0}^{1}\|\dot{\eta}_{t}\|\mathrm{dt}$ and $L(\gamma)\coloneqq\int_{0}^{1}\|\dot{\gamma}_{t}\|\mathrm{dt}$. We say that \emph{isometry} holds if for any path $\eta$ in $\mathcal{Z}$
and the path in $\gamma$ defined by $\gamma_{t}\coloneqq\phi(\eta_{t})$,
$L(\gamma)$ and $L(\eta)$ are equal up to a constant of proportionality.
In particular this implies that the length of any shortest path in
$\mathcal{Z}$ with end-points $z,z^{\prime}$ is equal to the length
of any shortest path in $\mathcal{M}$ with end-points $\phi(z),\phi(z^{\prime})$. It is known that isometry holds under various sufficient conditions on the  kernel function $(z,z^{\prime})\mapsto\mathbb{E}[\mathbf{Y^{\mathrm{nf}}}(z)\cdot\mathbf{Y}^{\mathrm{nf}}(z^{\prime})]$ \cite{whiteley2021matrix,whiteley2022statistical}. However, in practice this kernel function will often be unknown.
We can nevertheless assess evidence of isometry given $\*Y_n$ and $\*Z_n$ in a manner inspired by the manifold learning literature, specifically the first step of Isomap \cite{tenenbaum2000global}. 

We approximate shortest
path-lengths in $\mathcal{Z}$ by constructing a $k$-nearest\footnote{a default choice of $k$ is to choose it to be the smallest value such that the graph is connected} neighbour graph with vertex set $\*Z_n$
and edge-lengths $\|z_{i}-z_{j}\|$ if $z_{i}$ and $z_{j}$ are neighbours;
then compute shortest paths in this edge-weighted graph. To approximate
shortest path-lengths in $\mathcal{M}$, we similarly construct a $k$-nearest neighbour graph with vertex set $\*Y_n$
and edge lengths $\|\mathbf{Y}_{i}-\mathbf{Y}_{j}\|$. We write $\widehat{L}(z_{i},z_{j})$
and $\widehat{L}(\mathbf{Y}_{i},\mathbf{Y}_{j})$ for the resulting shortest path-lengths, noting that these lengths 
of the form $\sum_{m}\|z_{k_{m}}-z_{k_{m-1}}\|$ and $\sum_{m}\|\mathbf{Y}_{\ell_{m}}-\mathbf{Y}_{\ell_{m-1}}\|$
for some $[n]$-valued sequences $(k_{m})$ and $(\ell_{m})$. If isometry holds with constant of proportionality $\beta$, then when $z$ is close to $z^\prime$, $\|\phi(z)-\phi(z^\prime)\| \approx \beta \|z-z^\prime\|$. This suggests assessing evidence of isometry through linear regression of the values $(\widehat{L}(\mathbf{Y}_{i},\mathbf{Y}_{j});i\neq j)$
on to the associated values $(\widehat{L}(z_{i},z_{j});i\neq j)$ allowing for a positive intercept to account for noise (informed by \eqref{eq:distance}). The correlation coefficient, $\rho$, associated with this linear regression numerically quantifies the presence of isometry, with a maximum value of $\rho=1$ corresponding to a perfect linear fit.


\vspace{-0.2cm}
\section{Manifold analysis of grid cell activity}\label{sec:neuro}
\vspace{-0.2cm}
\paragraph{Context: Gardner et al.'s discovery of toroidal structure in grid cell activity.} Understanding how spatial location is represented in brain activity  is a fundamental neuroscience challenge. In 2014, May-Britt Moser and Edvard I. Moser received a Nobel prize for their discovery of \emph{grid cells} -- nerve cells whose firing activity is associated with a grid of spatial locations \cite{fyhn2004spatial,hafting2005microstructure}. Combined with earlier work of co-prize-winner John \citet{o1976place}, this discovery showed that the brain effectively creates a map of the world around us  through the firing patterns of neurons.   In their 2022 \emph{Nature} paper, \citet{gardner2022toroidal} established a deeper understanding of grid cell  activity using persistent homology techniques, revealing the presence of toroidal structure. We focus on a subset of the experiments they reported, in which a rat was confined to a square enclosure (``open field''), and the firing of $p\approx 150$ grid cells was co-recorded and sub-sampled into $n=15,000$
time bins of length 10ms, while simultaneously recording the rats physical position in the enclosure, which we denote $\xi_1,\ldots,\xi_n$. Representing the neural activity data as $\Y_1,\ldots,\Y_n$, the $j$th element of $\Y_i$ is the firing rate of the $j$th grid cell in the $i$th time bin. After pre-processing  $\Y_1,\ldots,\Y_n$ (details in supplementary material), \cite{gardner2022toroidal} made the following two key findings: 
\begin{enumerate}[topsep=0pt,itemsep=-1ex,partopsep=1ex,parsep=1ex,leftmargin=0.5cm]
\item Persistent homology analysis of $\Y_1,\ldots,\Y_n$ gave estimated Betti numbers of $H_0=1$ (one connected component), $H_1=2$ (two 1D `holes' ) and $H_2=1$ (one 2D cavity), indicating presence of a torus.
\item  Circular coordinates corresponding to the two 1D holes in the torus were found to be associated with physical locations through a two-step transformation: firstly, ``unwrapping'' the surface of the torus to form a rhombus (moving along one edge of the rhombus corresponds to moving $0-360^\circ$ degrees around one of the 1D holds in the torus); and secondly,  covering the physical enclosure with a tessellation of this rhombus. 
\end{enumerate}
Using the data and code of \citet{gardner2022toroidal}, available at \cite{gardner_data, gardner_code}, we re-created this analysis. The results are shown in figure \ref{fig:torus}.    
\begin{figure}
    \centering
    \includegraphics[width=0.9\linewidth]{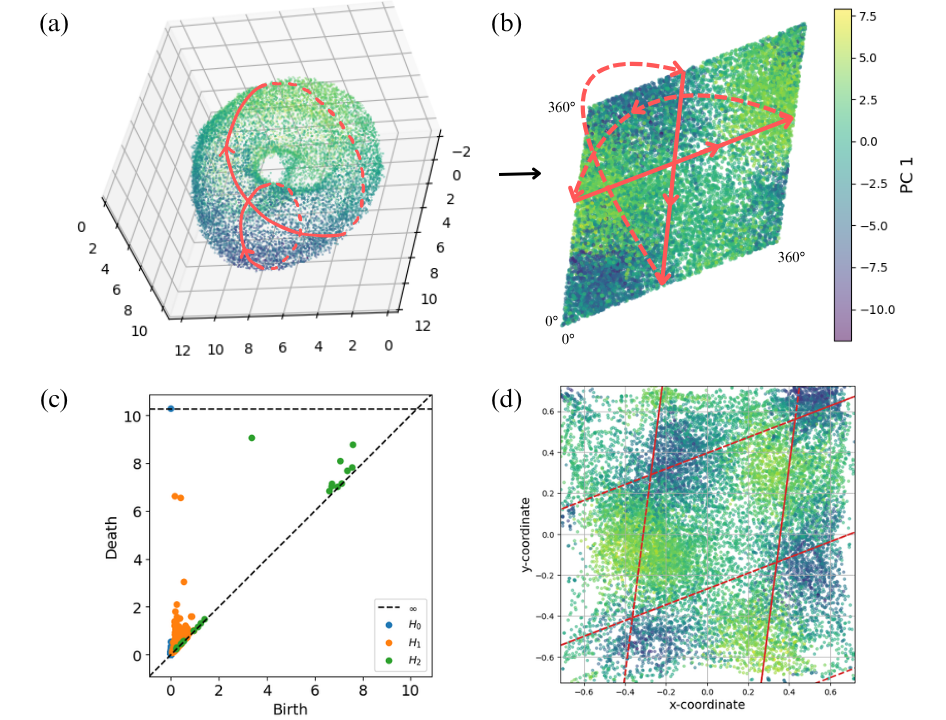}
    \vspace{-5pt}
    \caption{Re-creating the grid cell analysis of \citet{gardner2022toroidal}. (a): UMAP visualisation of the grid cell data $\Y_1,\ldots,\Y_n$. This visualisation suggest the presence of toroidal  structure, confirmed by the persistence diagram in (b), indicating Betti numbers $H_0=1$, $H_1=2$, $H_2=1$. In (c), cohomological decoding maps the circular coordinates of the torus to a rhombus. (d) shows how these coordinates correspond to physical space through tesselation of the rhombus. In (a), (c) and (d), points are colored  by the first component in the PCA embedding of the data to aid visual recognition of the torus.
    }
    \label{fig:torus}
    \vspace{-15pt}
\end{figure}

\paragraph{Our contribution: revisiting Gardner et al.'s data in search of isometry}

The analysis of \citet{gardner2022toroidal} indicates how the circular on-torus coordinates align with x-y coordinates in physical space, but we take this one step further, asking:\vspace{-0.2cm}
\begin{center}is there \emph{isometry} between the toroidal structure found in $\Y_1,\ldots,\Y_n$ and physical space?
\end{center}\vspace{-0.2cm}
If so, this would extend the topological findings of \cite{gardner2022toroidal}, indicating that the firing activity of grid cells conveys a geometrically faithful map of the world around us. To frame this question, we consider the grid-cell firing data $\Y_1,\ldots,\Y_n$  to be generated from  three alternative instances of the model from section \ref{sec:data_model}, each of which is built upon a different choice of metric space $(\mathcal{Z},d_{\mathcal{Z}})$ used to represent physical locations. In each case, we then look for evidence of isometry between $\mathcal{Z}$ and $\mathcal{M}$.



\textbf{Model 1: Open field, with real-world Euclidean distance.}
$\*Z$ is the open field of all real-world physical locations the rat could have possibly occupied, specifically the square $[-0.75,0.75]^2\subset \mathbb{R}^2$ and $z_i=\xi_i$ is the physical location of the rat at the $i$th time point. The metric $d_{\mathcal{Z}}$ is Euclidean distance, i.e. straight-line distance. Making no use of the topological analysis in \cite{gardner2022toroidal}, this model can be regarded as the default hypothesis for how physical geometry relates to the geometry of $\mathcal{M}$ and hence $\mathcal{Y}_n$.

\textbf{Model 2: Superimposition on rhombus, with Euclidean distance.}
$\*Z$ is the central rhombus in the tesselation of  $[-0.75,0.75]^2$ obtained by \citet{gardner2022toroidal} (see figure \ref{fig:path_analysis} (b), left plot)  and $z_i$ is the superimposition onto this rhombus of the physical location of the rat at the $i$th time point (figure \ref{fig:path_analysis} (b), middle plot). The metric $d_{\*Z}$ is Euclidean distance on the rhombus. Unlike Model 1, Model 2 accounts for the tessellated rhomboidal pattern in the firing data as a function of physical space.

\textbf{Model 3: Superimposition on rhombus, with Euclidean teleportation distance.}
$\*Z$ and $z_i$ are the same as in Model 2, but $d_{\*Z}$ is Euclidean distance subject to periodic boundary conditions on the rhombus, i.e., points on opposing edges of the rhombus correspond to the same position (same circular coordinates) and therefore paths on $\*Z$ can `teleport' at the edges to the corresponding point on the opposite edge. Defining $d_{\*Z}$ in this way extends Model 2 to respect toroidal topology. 


In Models 2 and 3, the choice of superimposing on the central rhombus rather than any of the others is arbitrary -- all that matters is the superimposition of physical locations on to one common rhombus. 
The difference between the three models is illustrated in figure \ref{fig:path_analysis}(b). Let $\xi_i,\xi_j$ be two physical locations which the rat occupied at some points $i,j$ in time,
indicated by respectively blue and green dots in the left plot of figure \ref{fig:path_analysis}(b). Under Model 1, where $z_i = \xi_i$ and $ z_j=\xi_j$, the shortest path in $\*Z_n$ between these two points is shown in green. In Model 2, $z_i$ and $z_j$ are the superimposition of $\xi_i$ and $\xi_j$ onto the central rhombus; for the particular points in the example of \ref{fig:path_analysis}(b) middle plot, we have $\xi_i= z_i$ (blue), but $\xi_j\neq  z_j$ (green). The shortest path in $\*Z_n$ is shown in green. In Model 3, $\xi_i$ and $\xi_j$ are represented by the same points $z_i$ and $z_j$ as in Model 2, but the definition of distance and hence shortest path is different because of `teleporting', represented by a dashed line in the right plot of \ref{fig:path_analysis}(b), is allowed.

Figure \ref{fig:path_analysis}c) shows comparisons of shortest path-lengths in $\*Z_n$ and in $\*Y_n$ under Models 1-3. As per section \ref{sec:isometry}, an isometric relationship between $\*Z$ and $\*M$ in any one of these models would manifest itself in a linear relation between shortest path-lengths.  For each model, orange lines show straight line fit from OLS with corresponding correlation coefficients, $\rho$, reported, red lines show moving averages (with window size 0.01$n$), and shading indicates $\pm 1$s.d. Model 3 clearly achieves the best linear fit. We believe that the deviation from linear in the top-right is likely due to error in estimation of the vectors defining the rhombus.

In conclusion: we find that grid cell activity conveys a faithful representation of distances in the real world, where the neural representations live close to a toroidal structure, embedded in high-dimensional space. Put more concretely: we find evidence that $(\mathcal{Z}, d_\mathcal{Z})$, as defined in Model 3, is \textit{isometric} to the neural manifold $\mathcal{M}$, whose toroidal topology is confirmed by persistent homology analysis.

\begin{figure}[]
    \centering
    \includegraphics[width=0.88\linewidth]{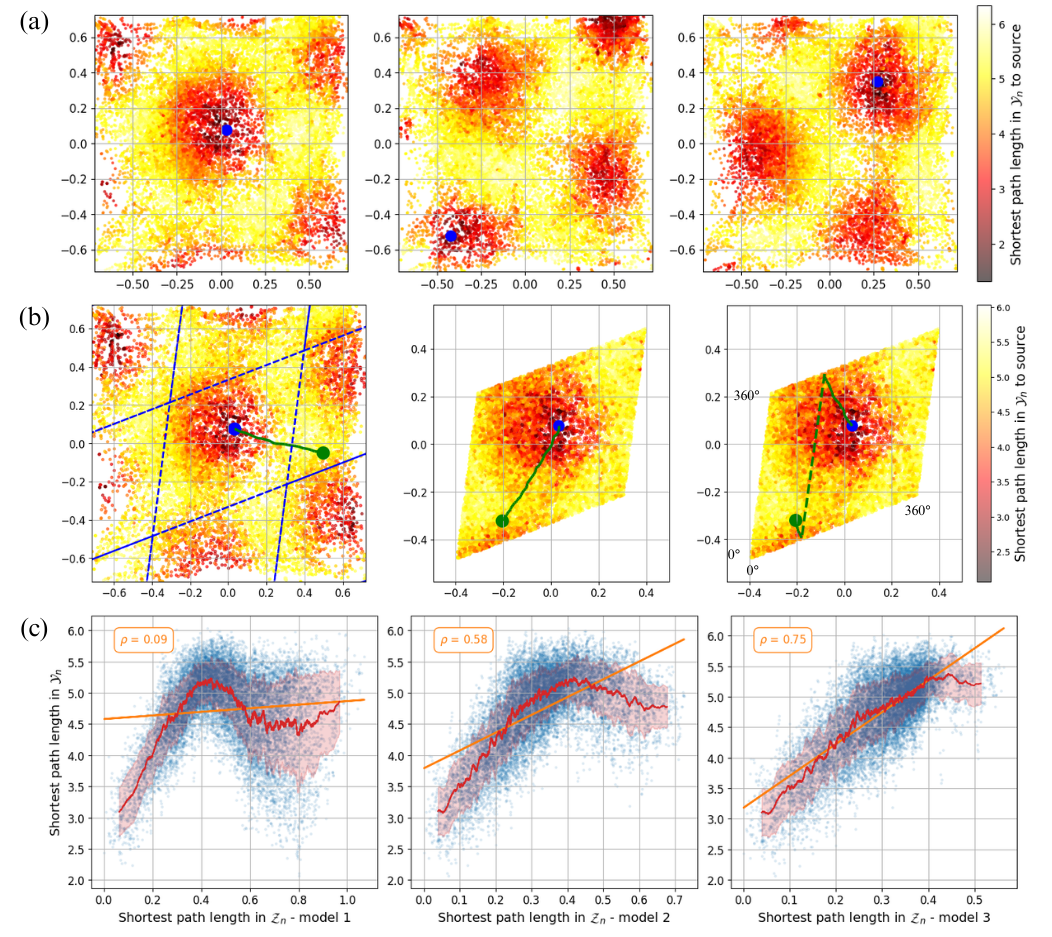}
    \vspace{-7pt}
    \caption{\textbf{(a)} For 3 possible source locations shown in blue, all other physical locations visited by the rat are colored by the shortest path-length in $\mathcal{Y}_n$ from the source. \textbf{(b)} \textit{Left:} The tessellated rhombus is plotted atop the physical locations. Two physical locations (blue, green), and the shortest path in $\*Z_n$ between them under Model 1 are shown. \textit{Middle}: Representations of the same source and sink locations under Model 2, in which physical locations are superimposed on the central rhombus in the tesselation. Shortest path in $\*Z_n$ under Model 2 is shown. \textit{Right}: Under Model 3, the distance $d_{\mathcal{Z}}$ allows for `teleporting' (dashed line). \textbf{(c)} relationships between shortest path-lengths in $\*Y_n$ and in $\*Z_n$ for Models 1-3 (left to right). Orange lines show best linear fit and red lines show moving averages (shading for $\pm 1$s.d.). Strongest evidence of isometry appears for Model 3.}    
    \label{fig:path_analysis}
    \vspace{-10pt}
\end{figure}

\vspace{-0.3cm}
\section{Limitations and outlook}\vspace{-0.3cm}The generic structure of our statistical model does not constrain us to one or other specific parametric family of distributions. However, our $\sqrt{\log n/ p_{\mathrm{int}}}$ convergence rates are directly linked to the sub-Gaussian assumptions we make. Weakening to sub-exponential or polynomial moment conditions would result in slower convergence rates, $\log n$ being replaced by some faster growing function of $n$, connecting back to the HDLSS setting of \cite{hall2005geometric} under their very mild $4$-th moment conditions.  However, if one were to make specific distributional assumptions (e.g. Gaussian, as in the Gaussian Process Latent Variable Model of \cite{lawrence2003gaussian,lawrence2005probabilistic}) then, in principle, one could pursue parametric, likelihood-based inference, assess model fit, etc. We view this as an important but different task to conventional TDA and manifold learning as considered here. 

We don't consider there to be any direct societal impacts of our work, but one of the key messages conveyed by our theoretical results: that  $p\gg n$ is not required for latent geometry to emerge in data, potentially broadens the applicability of manifold learning methods across application domains. Another key message is that the limiting ``close-to-orthogonal'' geometry of data in proposition \ref{prop:iid} is just one possibility; theorem \ref{thm:rand_func} and proposition \ref{prop:rand_func} imply a huge range of alternatives exist. We think the limiting data geometry in proposition \ref{prop:iid} and similar `thin-shell' concentration phenomena have entered into Machine Learning folklore; they are often relied upon when arguing that learning methods suffer from the \emph{curse of dimensionality}, see e.g., the textbook of \citet{hastie01}[Sec. 2.5]. The present work provides a mathematical framework in which such arguments might, in future, be re-assessed.

\newpage
\bibliographystyle{plainnat}
\bibliography{refs}


\newpage
\appendix

\section*{Appendices}

\section{Code}

Python notebooks for this paper are available at:
\url{https://github.com/hsansford1}

\section{Related Work}

\paragraph{Curse versus blessing of dimension.}
The phrase ``\emph{curse of dimensionality}'' is often used to convey the general idea that a large number of features, $p$, can be statistically or computationally problematic for a variety of learning methods. Amongst several manifestations of this issue, \citet[Sec 2.5]{hastie01} highlight that $n$ i.i.d. samples from the uniform distribution on the mean-zero $p$-dimensional unit ball tend to all like close to its surface as $p\to\infty$ grows. Refined analysis of this phenomenon is developed in \cite{cai2013distributions}. Our work contrasts with this perspective: in the setting of the random function model we introduce, $p_{\mathrm{int}}\to\infty$ leads to the emergence of latent topology and manifold structure in data, arguably more of a  \emph{blessing of dimensionality} \citeapp{kainen11997utilizing,donoho2000high}. See \citeapp[Sec. 2.1]{lawrence2011spectral} for a discussion of similar considerations in the context of dimension reduction techniques.  

\paragraph{The Manifold Hypothesis.}
The manifold hypothesis asserts that nominally high-dimensional data from the real world are concentrated on or near a low-dimensional set embedding in high-dimensional space \citeapp{cayton2005algorithms}, \cite{bengio2013representation,fefferman2016testing}. This phenomenon is the motivation for intrinsic dimension estimation methods, e.g., \citeapp{kegl2002intrinsic,levina2004maximum,levina2004maximum,hein2005intrinsic,carter2009local,little2011estimating} and wide range of nonlinear dimension reduction techniques, including:  \cite{tenenbaum2000global,van2008visualizing, mcinnes2018umap}.

\citet{whiteley2022statistical} proposed a generic form of latent variable model to explain manifold structure in data. The random function model considered in the present paper is a variation on the model of \citet{whiteley2022statistical}, with the following differences:
\begin{itemize}[topsep=0pt,itemsep=0ex,partopsep=1ex,parsep=1ex,leftmargin=0.5cm]
\item In the present paper we parameterise our model in terms of the covariance matrices $\SS(\cdot)$ and mean vectors $\boldsymbol{\mu}(\cdot)$, and it is this parameterization which enables us to explain the role of the ambient intrinsic dimensions $p_{\mathrm{int}}^{(i)}$. \citet{whiteley2022statistical} did not consider this parameterisation or the notion of ambient intrinsic dimension.
\item Connected with this parameterization, the kernel function we consider \eqref{eq:kernel_defn} differs from that of  \citet{whiteley2022statistical} in how it is normalised.
\item \citet{whiteley2022statistical} studied the statistical properties of the PCA embedding of data from their model, under assumptions including: i) finite rank $(r<\infty)$, ii) independence across features, iii) uniformly bounded fourth moments, and iv) latent variables $z_i$ being random, and iid. By contrast our theorem \ref{thm:rand_func} and proposition \ref{prop:rand_func} address not the PCA embedding for rather dot-products between data vectors, we make no assumption about $r$ being finite, we assume sub-Gaussianity, and do not make any statistical assumptions about the $z_i$. In this sense the assumptions of the present work are in some ways more general (we do not require independence) but in other ways stronger and leading to more refined results (we require sub-Gaussianity, and that leads to our $\sqrt{\log n /p_{\mathrm{int}}}$ convergence rates.)
\end{itemize}

\citet{whiteley2022statistical} explained homeomorphic (as per lemma \ref{lem:homeo}) and isometric properties of kernel feature maps, including giving sufficient conditions on the kernel function for isometry to hold, building on older work, \cite{whiteley2021matrix}.

Motivated by hierarchical clustering \citet{gray2023hierarchical} studied asymptotic behaviour of dot-products between data vectors as a function of $n$ and $p$ under a tree structured model. They assumed mixing across dimensions and a polynomial moment condition. The convergence rates they obtain are slower than those we obtain, reflecting our sub-Gaussian assumptions. The notion of ambient intrinsic dimension does not appear in \cite{gray2023hierarchical}.

\paragraph{The Hanson-Wright inequality.}
The Hanson-Wright (HW) inequality is a concentration inequality for quadratic forms $\X^\top \A \X$, where $\X$ is a vector of independent, sub-Gaussian random variables. A version of this theorem was first published in \citep{hanson1971bound,wright1973bound}.  The result has been revisited under various sets of assumptions surveyed by \citet[Rem. 1.2]{rudelson2013}, who presented their own proof using tools from high-dimensional probability. Our GHW inequality concerns concentration of $\+X^\top \+A \+X^\prime$, where the random vectors $\X=(X_1,\ldots,X_p)$ and $\X^\prime=(X^\prime_1,\ldots,X^\prime_p)$ are such that the pairs $(X_j,X_j^\prime)$ are independent across $j=1,\ldots,p$. Our proof follows the strategy of \citet[Thm 1.1] {rudelson2013} very closely and our inequality reduces to exactly their result when $\X = \X^\prime$. At a technical level, the generality achieved in our proof is therefore modest, but in terms of interpretation and applicability, our result presents a substantial step forward because it opens the door to understanding the concentration behaviour of dot products among point clouds using the GHW inequality.

\paragraph{HDLSS asymptotics.}
The theory of high-dimension, low sample size (HDLSS) asymptotics \cite{hall2005geometric,ahn2007high,shen2016statistics,aoshima2018survey} has made a transformative impact on our understanding and intuition for how random vectors behave in high-dimensions. The earliest of these works treated the case where $n$ is fixed and $p\to\infty$, whilst noting that the case $p/n^2\to\infty$ could be treated in a similar fashion. This reflects the  $4$-th moment conditions assumed in \cite{hall2005geometric}, which is considerably weaker than sub-Gaussianity.  The significance of the regime $p/n^2\to\infty$ can be illustrated as follows.

Let $\X=(X_1,\ldots,X_p)$ be a zero-mean random vector with i.i.d elements satisfying $\E[|X_j|^2]=1$ and $\E[|X_j|^4]<\infty$, and let  $\X^\prime$ be an independent copy of $\X$. Then $\E[\|\X\|^2]=p$ and $\E[\X\cdot\X^\prime]=0$. Noting that $\|\X\|^2-p$ and   $\X\cdot\X^\prime$ are both sums of independent, mean-zero random variables, Chebychev's inequality implies:
$$
\P\left(\left|\frac{\|\X\|^2}{p}-1\right|>\delta\right)\;\vee \; \P\left(\left|\frac{\X\cdot\X^\prime}{p}\right|>\delta\right)\leq \frac{\E[|X_1|^4]}{\delta^2 p}.
$$
Via a union bound, it follows that if $\X_1,\ldots,\X_n$ are independent copies of $\X$, then 
$$
\max_{i,j\in[n]}\left|\frac{1}{p} \X_i\cdot\X_j - \mathbf{I}[i=j]\right| \in O_{\P}\left(\frac{n}{\sqrt{p}}\right)
$$
for $p$ growing as $n\to\infty$, cf. proposition \ref{prop:iid}. The sub-Gaussian setting of  proposition \ref{prop:iid} is far stronger than the finite fourth moment condition invoked above. The price to pay for the finite fourth moment condition is the slower $n/\sqrt{p}$ convergence rate. 

\paragraph{Under sub-Gaussian assumptions, when does $p\gg n$ really matter?}
The convergence rates we obtain in the present paper are closely tied to the combination of our sub-Gaussian assumptions and the $\max_{i,j\in[n]}$ appearing in proposition \ref{prop:iid} and theorem \ref{thm:rand_func}. Let use define the matrix $\Y =[\Y_1|\cdots | \Y_n]^\top\in\mathbb{R}^{n\times n}$. Assuming for ease of exposition that $\E[\|\Y_i\|^2]=p$ and, our results about dot products $\Y_i\cdot\Y_j$ as in proposition \ref{prop:iid} and theorem \ref{thm:rand_func} can be re-written as controlling $p^{-1}\|\Y\Y^\top-\E[\Y \Y^\top]\|_{\mathrm{max}}$, where for a matrix $\mathbf{A}=(A_{ij})$, $\|\mathbf{A}\|_{\mathrm{max}} = \max_{ij}|A_{ij}|$. 

Under sub-Gaussian assumptions, other matrix norms of $\Y\Y^\top-\E[\Y \Y^\top]$ can scale differently with $n$ and $p$, notably the spectral norm. For example, with $\tilde{\Y}_j$ denoting the $j$th column of $\Y$ so that $p^{-1}\Y\Y^\top\equiv p^{-1}\sum_{j=1}^p\tilde{\Y}_j\tilde{\Y}_j^\top $, if the vectors $\tilde{\Y}_j$ are i.i.d. and  uniformly bounded as $\|\tilde{\Y}_j\|^2\leq \text{const.} \cdot n$, then by \cite[Sec 1.6.3.]{tropp2015introduction}, $p \sim n \log n$ is sufficient to control the relative error $\E[\|\Y\Y^\top-\E[\Y\Y^\top]\|]/ \|\E[\Y\Y^\top]\|$, and this condition is, in the words of \cite[Sec 1.6.3.]{tropp2015introduction}, ``qualitatively sharp for worst case distributions''.

We note that the condition ``$p\gg n$'' is also famously associated with sparse regression problems in which the number of available covariates, $p$, is much larger than the number of samples, $n$, as can be approached using the Lasso \citeapp{tibshirani1996regression} and similar techniques. Those sparse regression problems seem not to be closely related to the geometric considerations in the present work.

\section{Proofs and supporting results for section \ref{sec:HS}}

The sub-exponential norm of a random variable $X$ is:
\[
\|X\|_{\psi_{1}}\coloneqq\sup_{q\geq1}q^{-1}\mathbb{E}[|X|^{q}]^{1/q}.
\]
\begin{proof}[Proof of theorem \ref{thm:HW_inequality}]
The proof follows the arguments of \citet[Proof of thm 1.1]{rudelson2013} very closely. As they did, by replacing $\X$ with $\X/K$ we can assume without loss of generality that $K=1$.

Our first objective is to estimate 
\[
\pi\coloneqq\mathbb{P}\left(\mathbf{X}^{\top}\mathbf{A}\mathbf{X}^{\prime}-\mathbb{E}\left[\mathbf{X}^{\top}\mathbf{A}\mathbf{X}^{\prime}\right]>t\right).
\]
Let $\mathbf{A}=(a_{ij})_{i,j=1}^{p}$. The independence assumption of  the theorem implies that $X_i$ and $X_j^\prime$ are independent for $i\neq j$. Using this together with the zero-mean assumption of the theorem,
\begin{align*}
\mathbf{X}^{\top}\mathbf{A}\mathbf{X}^{\prime}-\mathbb{E}\left[\mathbf{X}^{\top}\mathbf{A}\mathbf{X}^{\prime}\right] & =\sum_{i,j}a_{ij}X_{i}X_{j}^{\prime}-\sum_{i}a_{ii}\mathbb{E}[X_{i}X_{i}^{\prime}]\\
 & =\sum_{i}a_{ii}\left(X_{i}X_{i}^{\prime}-\mathbb{E}[X_{i}X_{i}^{\prime}]\right)+\sum_{i\neq j}a_{ij}X_{i}X_{j}^{\prime}.
\end{align*}
The problem is split up into estimating deviation probabilities associated
with the diagonal and off-diagonal sums:
\[
\pi\leq\mathbb{P}\left(\sum_{i}a_{ii}\left(X_{i}X_{i}^{\prime}-\mathbb{E}[X_{i}X_{i}^{\prime}]\right)>t/2\right)+\mathbb{P}\left(\sum_{i\neq j}a_{ij}X_{i}X_{j}^{\prime}>t/2\right)=:\pi_{1}+\pi_{2}.
\]

Note that $X_{i}X_{i}^{\prime}-\mathbb{E}[X_{i}X_{i}^{\prime}]$ are
independent, mean-zero, sub-exponential random variables, indeed for
$q\geq1$, by Minkowski's, Jensen's and the Cauchy-Schwartz inequalities,
\begin{multline*}
\mathbb{E}\left[\left|X_{i}X_{i}^{\prime}-\mathbb{E}[X_{i}X_{i}^{\prime}]\right|^{q}\right]^{1/q}\leq\mathbb{E}\left[\left|X_{i}X_{i}^{\prime}\right|^{q}\right]^{1/q}+\left|\mathbb{E}[X_{i}X_{i}^{\prime}]\right|\\
\leq2\mathbb{E}\left[\left|X_{i}X_{i}^{\prime}\right|^{q}\right]^{1/q}\leq2\mathbb{E}\left[\left|X_{i}\right|^{2q}\right]^{1/2q}\mathbb{E}\left[\left|X_{i}^{\prime}\right|^{2q}\right]^{1/2q}
\end{multline*}
and hence
\[
\|X_{i}X_{i}^{\prime}-\mathbb{E}[X_{i}X_{i}^{\prime}]\|_{\psi_{1}}\leq2\|X_{i}X_{i}^{\prime}\|_{\psi_{1}}\leq4\|X_{i}\|_{\psi_{2}}\|X_{i}^{\prime}\|_{\psi_{2}}\leq4.
\]
By a Bernstein-type inequality \cite[Prop 5.16]{vershynin2010introduction}, 

\[
\pi_{1}\leq\exp\left[-c\min\left\{ \frac{t^{2}}{\|\mathbf{A}\|_{\mathrm{F}}^{2}},\frac{t}{\|\mathbf{A}\|}\right\} \right].
\]

It remains to consider the term:

\[
S\coloneqq\sum_{i\neq j}a_{ij}X_{i}X_{j}^{\prime}.
\]
For $\lambda>0$ we have 
\[
\pi_{2}=\mathbb{P}\left(S>t/2\right)\leq\exp(-\lambda t/2)\mathbb{E}\left[\exp(\lambda S)\right].
\]
Let $\delta_{1},\ldots,\delta_{p}$ be independent $\{0,1\}$-valued
random variables with $\mathbb{E}[\delta_{i}]=1/2$. Then 
\[
S=4\mathbb{E}_{\delta}[S_{\delta}],\quad\text{where}\quad S_{\delta}\coloneqq\sum_{i,j}\delta_{i}(1-\delta_{j})a_{ij}X_{i}X_{j}^{\prime},
\]
and where $\mathbb{E}_{\delta}$ denotes expectation with respect the
distribution of $\delta=(\delta_{1},\ldots,\delta_{p})$. By Jensen's
inequality,
\[
\mathbb{E}\left[\exp(\lambda S)\right]\leq\mathbb{E}\left[\exp(4\lambda S_{\delta})\right].
\]
 With $\Lambda_{\delta}\coloneqq\{i\in[p]:\delta_{i}=1\}$, we have

\[
S_{\delta}=\sum_{j\in\Lambda_{\delta}^{c}}X_{j}^{\prime}\left(\sum_{i\in\Lambda_{\delta}}a_{ij}X_{i}\right).
\]
Conditional on $\delta$ and $(X_{i})_{i\in\Lambda_{\delta}}$, $S_{\delta}$
is a linear combination of mean-zero, sub-Gaussian random variables
$X_{j}^{\prime}$, $j\in\Lambda_{\delta}^{c}$, with associated coefficients
$\sum_{i\in\Lambda_{\delta}}a_{ij}X_{i}$. It follows that the conditional
distribution of $S_{\delta}$ is sub-Gaussian with (conditional) sub-Gaussian
norm:
\[
\|S_{\delta}\|_{\psi_{1}}\leq C\sigma_{\delta},\qquad\sigma_{\delta}^{2}\coloneqq\sum_{j\in\Lambda_{\delta}^{c}}\left(\sum_{i\in\Lambda_{\delta}}a_{ij}X_{i}\right)^{2},
\]
for some constant $C>0$. 

Noting that $\sigma_\delta^2$ is a function of $\X$ but not of $\X^\prime$, the remaining steps in the proof follow the same arguments as in \citet[Proof of thm 1.1]{rudelson2013}, so the details are omitted.
\end{proof}

\begin{proof}[Proof of proposition \ref{prop:iid}]
Under the assumptions of the proposition we have $\mathbb{E}[\|\mathbf{Y}_{i}\|^{2}]=\mathrm{tr}\boldsymbol{\Sigma}$
and $\mathbb{E}[ \mathbf{Y}_{i}\cdot\mathbf{Y}_{j}]=0$
for $i\neq j$, hence 
\[
\pi_{ij}(t)\coloneqq\mathbb{P}\left(\left|\frac{  \mathbf{Y}_{i}\cdot\mathbf{Y}_{j} }{\mathrm{tr}\boldsymbol{\Sigma}}-\mathbf{I}[i=j]\right|>t\right)=\mathbb{P}\left(\left|\mathbf{X}_{i}^{\top}\boldsymbol{\Sigma}\mathbf{X}_{j} -\mathbb{E}[  \mathbf{X}_{i}^{\top}\boldsymbol{\Sigma}\mathbf{X}_{j} ]\right|>t\cdot\mathrm{tr}\boldsymbol{\Sigma}\right).
\]
 Applying theorem \ref{thm:HW_inequality} gives
\[
\pi_{ij}(t)\leq2\exp\left[-c\min\left\{ \frac{t^{2}(\mathrm{tr}\boldsymbol{\Sigma})^{2}}{K^{4}\mathrm{tr}(\boldsymbol{\Sigma}^{2})},\frac{t\cdot\mathrm{tr}\boldsymbol{\Sigma}}{K^{2}\|\boldsymbol{\Sigma}\|}\right\} \right],
\]
and by a union bound
\begin{align}
&\mathbb{P}\left(\max_{i,j\in[n]}\left|\frac{ \mathbf{Y}_{i}\cdot\mathbf{Y}_{j}}{\mathrm{tr}\boldsymbol{\Sigma}}-\mathbf{I}[i=j]\right|>t\right)\nonumber\\ 
& \leq\sum_{i,j\in [n]}\pi_{ij}(t)\nonumber \\
 & \leq2n^{2}\exp\left[-c\min\left\{ \frac{t^{2}(\mathrm{tr}\boldsymbol{\Sigma})^{2}}{K^{4}\mathrm{tr}(\boldsymbol{\Sigma}^{2})},\frac{t\cdot\mathrm{tr}\boldsymbol{\Sigma}}{K^{2}\|\boldsymbol{\Sigma}\|}\right\} \right]\nonumber \\
 & =2\exp\left[-c\min\left\{ \frac{t^{2}(\mathrm{tr}\boldsymbol{\Sigma})^{2}}{K^{4}\mathrm{tr}(\boldsymbol{\Sigma}^{2})},\frac{t\cdot\mathrm{tr}\boldsymbol{\Sigma}}{K^{2}\|\boldsymbol{\Sigma}\|}\right\} +2\log n\right].\label{eq:iid_union_bound_proof-1}
\end{align}
Let $M>0$ be a constant whose
value will be chosen later. Applying (\ref{eq:iid_union_bound_proof-1}) with $t=\sqrt{\frac{\log n}{p_{\mathrm{int}}}}M$,
\begin{align}
 & \mathbb{P}\left(\max_{i,j\in[n]}\left|\frac{\Y_i \cdot \Y_j }{\mathrm{tr}\boldsymbol{\Sigma}}-\mathbf{I}[i=j]\right|>\sqrt{\frac{\log n}{p_{\mathrm{int}}}}M\right)\nonumber \\
 & \leq2\exp\left[-c(\log n)\frac{M}{K^{2}}\min\left\{ \frac{M(\mathrm{tr}\boldsymbol{\Sigma})^{2}}{K\mathrm{tr}(\boldsymbol{\Sigma}^{2})}\frac{\|\boldsymbol{\Sigma}\|}{\mathrm{tr}\boldsymbol{\Sigma}},\sqrt{\frac{\mathrm{tr}\boldsymbol{\Sigma}}{\|\boldsymbol{\Sigma}\|}}\frac{1}{\sqrt{\log n}}\right\} +2\log n\right].\label{eq:iid_union_upper}
\end{align}
We have:
\[
\frac{(\mathrm{tr}\boldsymbol{\Sigma})^{2}}{\mathrm{tr}(\boldsymbol{\Sigma}^{2})}\frac{\|\boldsymbol{\Sigma}\|}{\mathrm{tr}\boldsymbol{\Sigma}}=\frac{(\mathrm{tr}\boldsymbol{\Sigma})}{\mathrm{tr}(\boldsymbol{\Sigma}^{2})}\|\boldsymbol{\Sigma}\|=\frac{\lambda_{1}\sum_{i}\lambda_{i}}{\sum_{i}\lambda_{i}^{2}}\geq1,
\]
where $\lambda_{1}\geq\lambda_{2}\geq\cdots$ are the eigenvalues
of $\boldsymbol{\Sigma}.$ Using the assumption that $\mathrm{tr}(\boldsymbol{\Sigma})/\|\boldsymbol{\Sigma}\|\in\Omega(\log n)$, there exists $c_{0}>0$ and $n_0\geq 1$ such that for $n\geq n_{0}$, 
\[
\sqrt{\frac{\mathrm{tr}\boldsymbol{\Sigma}}{\|\boldsymbol{\Sigma}\|}}\frac{1}{\sqrt{\log n}}\geq c_{0}.
\]
Hence for $n\geq n_{0}$, 
\begin{equation}
\min\left\{ \frac{M(\mathrm{tr}\boldsymbol{\Sigma})^{2}}{K\mathrm{tr}(\boldsymbol{\Sigma}^{2})}\frac{\|\boldsymbol{\Sigma}\|}{\mathrm{tr}\boldsymbol{\Sigma}},\sqrt{\frac{\mathrm{tr}\boldsymbol{\Sigma}}{\|\boldsymbol{\Sigma}\|}}\frac{1}{\sqrt{\log n}}\right\} \geq\min\left\{ \frac{M}{K},c_{0}\right\} .\label{eq:min_lower_bound}
\end{equation}
Now fix any $\epsilon>0$. By choosing $M$ large enough that $M/K\geq c_{0}$,
using (\ref{eq:min_lower_bound}), and then, increasing $M$ if necessary.
we can achieve: 
\begin{align*}
 & 2\exp\left[-c(\log n_{0})\frac{M}{K^{2}}\min\left\{ \frac{M}{K},c_{0}\right\} +2\log n_{0}\right]\\
 & \leq2\exp\left[-\log n_{0}\left(\frac{M}{K^{2}}cc_{0}-2\right)\right]\\
 & \leq\epsilon.
\end{align*}
Combining this inequality with (\ref{eq:iid_union_upper}) and (\ref{eq:min_lower_bound}), we have
for any $n\geq n_{0}$, 
\[
\mathbb{P}\left(\max_{i,j\in[n]}\left|\frac{\Y_i \cdot \Y_j }{\mathrm{tr}\boldsymbol{\Sigma}}-\mathbf{I}[i=j]\right|>\sqrt{\frac{\log n}{p_{\mathrm{int}}}}M\right)\leq\epsilon,
\]
which completes the proof.
\end{proof}

\section{Proofs and supporting results for section \ref{sec:data_model}}

Throughout the proofs in this section we will repeatedly use the fact that the square root of a symmetric, positive semidefinite  matrix $\mathbf{A}$ is a symmetric, positive semidefinite  matrix $\mathbf{B}$ such that $\mathbf{B}\mathbf{B} = \mathbf{B}\mathbf{B}^\top = \mathbf{A}$.

From the definition of $\mathbf{Y}_i$ in section \ref{sec:data_model} and \textbf{A\ref{ass:E}}, we have that
$$
\E[\|\Y_i\|^2] = \mathrm{tr}(\boldsymbol{\Sigma}(z_i)) + \|\boldsymbol{\mu}(z_i)\|^2 + p\sigma^2, 
$$
and 
$$
\E[\Y_i\cdot\Y_j] = \E[\Y^{\mathrm{nf}}(z_i)\cdot \Y^{\mathrm{nf}}(z_j) ] + \mathbf{I}[i=j]p\sigma^2. 
$$
Let us define 
\begin{align*}
C_{ij} &\coloneqq \E[\|\Y_i\|^2]^{1/2}\E[\|\Y_j\|^2]^{1/2} \\
&= \left( \mathrm{tr}(\boldsymbol{\Sigma}(z_i)) + \|\boldsymbol{\mu}(z_i)\|^2 + p\sigma^2\right)^{1/2}  \left( \mathrm{tr}(\boldsymbol{\Sigma}(z_j)) + \|\boldsymbol{\mu}(z_j)\|^2 + p\sigma^2\right)^{1/2},
\end{align*}
and $\mathbf{\Sigma}_{ij} \coloneqq \mathbf{\Sigma}(z_i)^{1/2}\mathbf{\Sigma}(z_j)^{1/2}$ (N.B., for each $i,j$, $\mathbf{\Sigma}_{ij}$  is a matrix).

We use the following lemmas \ref{lem:rand_func_first}-\ref{lem:rand_func_last} in our proof of theorem \ref{thm:rand_func}. These lemmas make rely on the same assumptions as theorem \ref{thm:rand_func}, but to avoid repetition we do not explicitly refer to these assumptions in the statements of the lemmas. 

In the statements and proofs of lemmas \ref{lem:rand_func_first}-\ref{lem:rand_func_last}, $M > 0$ is an arbitrarily chosen constant and 
$$t \coloneqq \sqrt{\frac{\log n}{\underset{i\in [n]}{\min}\left(p_\text{int}^{(i)}\right)}}M.$$

\begin{lem}
\label{lem:rand_func_first}
There exists $c_1, c_2 > 0$ and $n_0$ such that for $n \geq n_0$,
\begin{multline*}
\P\left(\left|\X(z_i)^\top\SS_{ij}\X(z_j) - \E[\X(z_i)^\top\SS_{ij}\X(z_j)] \right|>\frac{t}{8} C_{ij}\right)
\\ \leq 2 \exp \left( - (\log n) \frac{c_1M}{K^2} \min \left\{ \frac{M}{K^2}, c_2\right\} \right)
\end{multline*}
for all $i,j \in [n]$.
\end{lem}
\proof
We can bound the probability above using theorem \ref{thm:HW_inequality} with $\mathbf{A}=\mathbf{\Sigma}_{ij}$, noting $\mathbf{\Sigma}(z_i)^{1/2}, \mathbf{\Sigma}(z_j)^{1/2}$ are symmetric matrices, and using the cyclicity of trace, $\|\mathbf{\Sigma}_{ij}\|_{\mathrm{F}}^2 = \mathrm{tr}(\mathbf{\Sigma}_{ij}^\top \mathbf{\Sigma}_{ij})= \mathrm{tr}[\mathbf{\Sigma}(z_i)\mathbf{\Sigma}(z_j)]$, to get
\begin{align}
&\P\left(\left|\X(z_i)^\top\SS_{ij}\X(z_j) - \E[\X(z_i)^\top\SS_{ij}\X(z_j)] \right|>\frac{t}{8} C_{ij}\right) \nonumber\\
&\leq 2 \exp \left( - c_1 \min\left\{ \frac{t^2 C_{ij}^2}{K^4 \mathrm{tr}[\mathbf{\Sigma}(z_i)\mathbf{\Sigma}(z_j)]}, \frac{tC_{ij}}{K^2 \|\mathbf{\Sigma}_{ij}\|}\right\} \right).
\label{eq:hw_bound_xsx}
\end{align}
Then, focusing on the first term inside the minimum in \eqref{eq:hw_bound_xsx}, we use  
$C_{ij}\geq \mathrm{tr}(\boldsymbol{\Sigma}(z_i))^{1/2} \mathrm{tr}(\mathbf{\Sigma}(z_j))^{1/2} $,  $\mathrm{tr}[\mathbf{\Sigma}(z_i)\mathbf{\Sigma}(z_j)] \leq \left[ \mathrm{tr}(\mathbf{\Sigma}(z_i)^2) \mathrm{tr}(\mathbf{\Sigma}(z_j)^2) \right]^{1/2}$ and the following bound on $t$,
\begin{equation}
\label{eq:t_lower_bound} t \geq \sqrt{\frac{\log n}{\left(p_\text{int}^{(i)}p_\text{int}^{(j)}\right)^{1/2}}}M = \sqrt{\frac{\log n \|\mathbf{\Sigma}(z_i)\|^{1/2} \|\mathbf{\Sigma}(z_j)\|^{1/2} }{\mathrm{tr}(\mathbf{\Sigma}(z_i))^{1/2}\mathrm{tr}(\mathbf{\Sigma}(z_j))^{1/2}}} M,
\end{equation}
to obtain: 
\begin{align}
\frac{t^2 C_{ij}^2}{K^4 \mathrm{tr}[\mathbf{\Sigma}(z_i)\mathbf{\Sigma}(z_j)]} &\geq \frac{(\log n)M^2}{K^4} \frac{ \|\mathbf{\Sigma}(z_i)\|^{1/2}\|\mathbf{\Sigma}(z_j)\|^{1/2}\mathrm{tr}(\boldsymbol{\Sigma}(z_i))^{1/2} \mathrm{tr}(\mathbf{\Sigma}(z_j))^{1/2}}{ \mathrm{tr}[\mathbf{\Sigma}(z_i)\mathbf{\Sigma}(z_j)]} \nonumber\\
& \geq \frac{(\log n)M^2}{K^4} \sqrt{\frac{\lambda_1^{(i)}\lambda_1^{(j)} \left(\sum_k \lambda_k^{(i)} \right) \left(\sum_k \lambda_k^{(j)} \right)}{ \left(\sum_k {\lambda_k^{(i)}}^2 \right)\left(\sum_k {\lambda_k^{(j)}}^2 \right)}} \nonumber\\
&\geq \frac{(\log n)M^2}{K^4} ,\label{eq:t2C2bound}
\end{align}
where $(\lambda_k^{(i)})_{k\geq1}$ and $(\lambda_k^{(j)})_{k\geq1}$ are the eigenvalues of $\mathbf{\Sigma}(z_i)$ and $\mathbf{\Sigma}(z_j)$  respectively.

Now, for the second term inside the minimum in \eqref{eq:hw_bound_xsx}, using that $$C_{ij}\geq \mathrm{tr}(\boldsymbol{\Sigma}(z_i))^{1/2} \mathrm{tr}(\mathbf{\Sigma}(z_j))^{1/2} \geq  \mathrm{tr}(\mathbf{\Sigma}_{ij})$$,  
$\|\mathbf{\Sigma}(z_i)\|^{1/2}\|\mathbf{\Sigma}(z_j)\|^{1/2} \geq \|\mathbf{\Sigma}_{ij}\|$ and the bound on $t$ in \eqref{eq:t_lower_bound}, we get
\begin{align}
\frac{tC_{ij}}{K^2 \|\mathbf{\Sigma}_{ij}\|} &\geq \frac{\mathrm{tr}(\boldsymbol{\Sigma}(z_i))^{1/2} \mathrm{tr}(\mathbf{\Sigma}(z_j))^{1/2} M}{K^2 \|\mathbf{\Sigma}_{ij}\|}\sqrt{\frac{\log n \|\mathbf{\Sigma}(z_i)\|^{1/2} \|\mathbf{\Sigma}(z_j)\|^{1/2} }{\mathrm{tr}(\mathbf{\Sigma}(z_i))^{1/2}\mathrm{tr}(\mathbf{\Sigma}(z_i))^{1/2}}} \nonumber\\
&\geq  \frac{\sqrt{\log n}M}{K^2 }
\sqrt{\frac{\mathrm{tr}(\boldsymbol{\Sigma}(z_i))^{1/2} \mathrm{tr}(\mathbf{\Sigma}(z_j))^{1/2}}{\|\mathbf{\Sigma}(z_i)\|^{1/2} \|\mathbf{\Sigma}(z_j)\|^{1/2} }}.\label{eq:tC_bound}
\end{align}

Putting together \eqref{eq:t2C2bound} and \eqref{eq:tC_bound}, we have
\begin{multline*}
    \min\left\{ \frac{t^2 C_{ij}^2}{K^4 \mathrm{tr}[\mathbf{\Sigma}(z_i)\mathbf{\Sigma}(z_j)]}, \frac{tC_{ij}}{K^2 \|\mathbf{\Sigma}_{ij}\|}\right\} \\ \geq  (\log n) \frac{M}{K^2} \min \left\{ \frac{M}{K^2}, \sqrt{\frac{\mathrm{tr}(\boldsymbol{\Sigma}(z_i))^{1/2} \mathrm{tr}(\mathbf{\Sigma}(z_j))^{1/2}}{\|\mathbf{\Sigma}(z_i)\|^{1/2} \|\mathbf{\Sigma}(z_j)\|^{1/2} }} \frac{1}{\sqrt{\log n}}\right\}.    
\end{multline*}
 Using the assumption that $\min_{i\in[n]}\mathrm{tr}(\boldsymbol{\Sigma}(z_i))/\|\boldsymbol{\Sigma}(z_i)\|\in\Omega(\log n)$, there exists $c_{2}>0$ and $n_0$ such that for $n\geq n_{0}$, 
\begin{equation*}
\sqrt{\frac{\mathrm{tr}(\boldsymbol{\Sigma}(z_i))^{1/2} \mathrm{tr}(\mathbf{\Sigma}(z_j))^{1/2}}{\|\mathbf{\Sigma}(z_i)\|^{1/2} \|\mathbf{\Sigma}(z_j)\|^{1/2} }} \frac{1}{\sqrt{\log n}}
\geq c_{2}.
\end{equation*}
Hence, for $n\geq n_{0}$, 
\begin{equation*}
\    \min\left\{ \frac{t^2 C_{ij}^2}{K^4 \mathrm{tr}(\mathbf{\Sigma}_{ij}^2)}, \frac{tC_{ij}}{K^2 \|\mathbf{\Sigma}_{ij}\|}\right\} \geq(\log n) \frac{M}{K^2} \min \left\{ \frac{M}{K^2}, c_2\right\},
\end{equation*}
and substituting this into \eqref{eq:hw_bound_xsx} the result of the lemma follows.
\qed

\begin{lem}
\label{lem:hoeffding1}
There exists $c_3 > 0$ and $n_1$ such that for $n \geq n_1$
$$
\P\left(\left|\MMu(z_i)^\top \SS(z_j)^{1/2}\X(z_j)\right|>\frac{t}{8} C_{ij}\right)\leq  2 \exp\left(-\frac{c_3 M^2}{K^2} (\log n) \right).
$$
for all $i,j \in [n]$.
\end{lem}
\proof
 Since $\boldsymbol{\mu}(z_i)^\top \mathbf{\Sigma}(z_j)^{1/2}\mathbf{X}(z_j)$ is the sum of independent centered sub-gaussian random variables, its squared sub-gaussian norm is  bounded 
$$\|\boldsymbol{\mu}(z_i)^\top \mathbf{\Sigma}(z_j)^{1/2}\mathbf{X}(z_j)\|_{\psi_2}^2 \leq K^2 \| \boldsymbol{\mu}(z_i)^\top \mathbf{\Sigma}(z_j)^{1/2}\|_2^2 = K^2 \boldsymbol{\mu}(z_i)^\top \mathbf{\Sigma}(z_j)\boldsymbol{\mu}(z_i).$$ 
Therefore we can use a Hoeffding-type inequality for sub-gaussian random variables (see e.g. \cite[Prop. 5.10]{vershynin2010introduction} or \cite[Thm 2.6.3]{vershynin2018high}) to get
\begin{equation}
\label{eq:hoeffding1}
\P\left(\left|\MMu(z_i)^\top \SS(z_j)^{1/2}\X(z_j)\right|>\frac{t}{8} C_{ij}\right)\leq  2 \exp\left(-\frac{c_4}{K^2}\frac{t^2 C_{ij}^2}{\boldsymbol{\mu}(z_i)^\top \mathbf{\Sigma}(z_j)\boldsymbol{\mu}(z_i)} \right).
\end{equation}

Now, we use following lower bound on $C_{ij}^2$,
$$C_{ij}^2\geq \mathrm{tr}(\boldsymbol{\Sigma}(z_j))\|\boldsymbol{\mu}(z_i)\|_2^2, $$
 together with 
 \begin{equation*}t \geq \sqrt{\frac{\log n}{\left(p_\text{int}^{(i)}p_\text{int}^{(j)}\right)^{1/2}}}M = \sqrt{\frac{\log n \|\mathbf{\Sigma}(z_i)\|^{1/2} \|\mathbf{\Sigma}(z_j)\|^{1/2} }{\mathrm{tr}(\mathbf{\Sigma}(z_i))^{1/2}\mathrm{tr}(\mathbf{\Sigma}(z_j))^{1/2}}} M,
\end{equation*}
 to get
\begin{align*}
     \frac{t^2 C_{ij}^2}{\boldsymbol{\mu}(z_i)^\top \mathbf{\Sigma}(z_j)\boldsymbol{\mu}(z_i)} &\geq (\log n) M^2 \frac{\|\mathbf{\Sigma}(z_i)\|^{1/2}\|\mathbf{\Sigma}(z_j)\|^{1/2} \mathrm{tr}(\mathbf{\Sigma}(z_j)) \|\boldsymbol{\mu}(z_i)\|_2^2 }{\left[\mathrm{tr}(\mathbf{\Sigma}(z_i))\mathrm{tr}(\mathbf{\Sigma}(z_j))\right]^{1/2}\boldsymbol{\mu}(z_i)^\top \mathbf{\Sigma}(z_j)\boldsymbol{\mu}(z_i)} \nonumber \\
     &= (\log n) M^2 \left( \frac{\|\mathbf{\Sigma}(z_i) \|}{\mathrm{tr}(\mathbf{\Sigma}(z_i))} \right)^{1/2} \left(\frac{\mathrm{tr}(\mathbf{\Sigma}(z_j))} {\|\mathbf{\Sigma}(z_j) \|}\right)^{1/2} \frac{\|\mathbf{\Sigma}(z_j)\| \| \boldsymbol{\mu}(z_i)\|_2^2}{\sum_k \left\langle \boldsymbol{\mu}(z_j),U_k^{(i)} \right\rangle^2 \lambda_k^{(j)}} \nonumber \\
     & \geq (\log n) M^2 \left( \frac{\|\mathbf{\Sigma}(z_i) \|}{\mathrm{tr}(\mathbf{\Sigma}(z_i))} \right)^{1/2} \left(\frac{\mathrm{tr}(\mathbf{\Sigma}(z_j))} {\|\mathbf{\Sigma}(z_j) \|}\right)^{1/2},
\end{align*}
where $U_k^{(i)}$ is the eigenvector of $\mathbf{\Sigma}(z_i)$ associated with eigenvalue $\lambda_k^{(i)}$. Therefore, using the assumption that $\min_{i\in[n]}\mathrm{tr}(\boldsymbol{\Sigma}(z_i))/\|\boldsymbol{\Sigma}(z_i)\|\in\Omega(\log n)$,  there exists $c_{3}>0$ and $n_1$ such that for $n\geq n_{1}$, 
\begin{equation*}
    \frac{t^2 C_{ij}^2}{\boldsymbol{\mu}(z_i)^\top \mathbf{\Sigma}(z_j)\boldsymbol{\mu}(z_i)} \geq c_3M^2 (\log n).
\end{equation*}
Substituting this into \eqref{eq:hoeffding1}, gives the result.
\qed

\begin{lem}
There exists $c_4, c_5 > 0$ and $n_2$ such that for $n \geq n_2$
$$
\P\left(\left|\sigma\mathbf{E}_i^\top \SS(z_j)^{1/2}\X(z_j)\right|>\frac{t}{8} C_{ij}\right)\leq 2 \exp \left( -(\log n) \frac{c_5M}{K^2} \min \left\{ \frac{M}{K^2}, c_4\right \}\right)
$$
for all $i,j \in [n]$.
\end{lem}
\proof
First, we bound the above probability using theorem \ref{thm:HW_inequality} with $\mathbf{A} = \sigma \SS(z_j)^{1/2}$ to get
\begin{multline}
\label{eq:esx_bound}
    \P\left(\left|\sigma\mathbf{E}_i^\top \SS(z_j)^{1/2}\X(z_j)\right|>\frac{t}{8} C_{ij}\right)\\ \leq 2 \exp \left( -c_5 \min \left\{\frac{t^2 C_{ij}^2}{\sigma^2K^4 \mathrm{tr}[\SS(z_j)]}, \frac{t C_{ij}}{\sigma K^2 \|\SS(z_j)^{1/2}\|} \right\}\right).
\end{multline}
Now, focusing on the first term inside the minimum, we use that $C_{ij}^2 \geq p\sigma^2 \mathrm{tr}[\SS(z_j)]$ and $t \geq M\sqrt{(\log n ) / p}$
to get
\begin{align*}
    \frac{t^2 C_{ij}^2}{\sigma^2K^4 \mathrm{tr}[\SS(z_j)]} \geq \frac{M^2 \log n}{K^4}.
\end{align*}
Focusing now on the second term inside the minimum in \eqref{eq:esx_bound}, and using that $C_{ij} \geq p^{1/2} \sigma \mathrm{tr}[\SS (z_j)]^{1/2}$ and $t \geq M\sqrt{(\log n ) / p}$ we get
$$\frac{t C_{ij}}{\sigma K^2 \|\SS(z_j)^{1/2}\|} \geq \frac{M (\sqrt{\log n}) \mathrm{tr}[(\SS(z_j)]^{1/2}}{K^2 \|\SS(z_j)^{1/2} \|}$$
Therefore, we have that 
\begin{align*}
    \min \left\{\frac{t^2 C_{ij}^2}{\sigma^2K^4 \mathrm{tr}[\SS(z_j)]}, \frac{t C_{ij}}{\sigma K^2 \|\SS(z_j)^{1/2}\|} \right\} \geq  \frac{M}{K^2} (\log n)\min \left\{\frac{M}{K^2},  \frac{ \mathrm{tr}[(\SS(z_j)]^{1/2}}{ \|\SS(z_j)^{1/2} \|}\frac{1}{\sqrt{\log n}} \right \},
\end{align*}
and using the assumption that $\min_{i\in[n]}\mathrm{tr}(\boldsymbol{\Sigma}(z_i))/\|\boldsymbol{\Sigma}(z_i)\|\in\Omega(\log n)$, there exists $c_{4}>0$ and $n_2$ such that for $n\geq n_{2}$, 
$$\frac{ \mathrm{tr}[(\SS(z_j)]^{1/2}}{ \|\SS(z_j)^{1/2} \|}\frac{1}{\sqrt{\log n}} \geq c_4.$$
Hence, for $n \geq n_2$, 
$$ \min \left\{\frac{t^2 C_{ij}^2}{\sigma^2K^4 \mathrm{tr}[\SS(z_j)]}, \frac{t C_{ij}}{\sigma K^2 \|\SS(z_j)^{1/2}\|} \right\} \geq (\log n) \frac{M}{K^2} \min \left\{ \frac{M}{K^2}, c_4\right \}. $$
Substituting this into \eqref{eq:esx_bound} gives the result.
\qed

\begin{lem}
For all $i,j \in [n]:$
$$
\P\left(\left|\sigma\mathbf{E}_i^\top \MMu(z_j)\right|>\frac{t}{8} C_{ij}\right)\leq 2 \exp \left( - \frac{c_6 M^2}{K^2} \log n \right).
$$
\end{lem}
\proof
Using that $\|\sigma\mathbf{E}_i^\top \MMu(z_j) \|_{\psi_2}^2 \leq \sigma^2 K^2 \|\MMu(z_j) \|_2^2 $, we can bound the above probability using a Hoeffding-type equality for sub-gaussian random variables to get 
$$
\P\left(\left|\sigma\mathbf{E}_i^\top \MMu(z_j)\right|>\frac{t}{8} C_{ij}\right)\leq 2 \exp \left( - \frac{c_6 t^2 C_{ij}^2}{\sigma^2 K^2 \|\MMu(z_j) \|_2^2} \right).
$$
Then, using that $C_{ij}^2 \geq p \sigma^2\|\MMu(z_j)\|_2^2$ and $t^2 \geq M^2(\log n) /p$, we get
$$
\P\left(\left|\sigma\mathbf{E}_i^\top \MMu(z_j)\right|>\frac{t}{8} C_{ij}\right)\leq 2 \exp \left( - \frac{c_6 M^2}{K^2} \log n\right).
$$
\qed

\begin{lem}
\label{lem:rand_func_last}
There exists $c_7, c_8 > 0$ and $n_0$ such that for $n \geq n_3$
$$
\P\left(\left|\sigma^2 (\mathbf{E}_i\cdot\mathbf{E}_j -\E[\mathbf{E}_i\cdot\mathbf{E}_j]) \right|>\frac{t}{8} C_{ij}\right)\leq 2 \exp \left(- (\log n)\frac{c_7M}{K^2}\min \left \{\frac{M}{K^2}, c_8 \right\}\right)
$$
for all $i,j \in [n]$.
\end{lem}
\proof
First, we bound the above probability using theorem \ref{thm:HW_inequality} with $\mathbf{A} = \sigma^2 \mathbf{I}_p$ to get 
$$
\P\left(\left|\sigma^2 (\mathbf{E}_i\cdot\mathbf{E}_j -\E[\mathbf{E}_i\cdot\mathbf{E}_j]) \right|>\frac{t}{8} C_{ij}\right)\leq 2 \exp \left( -c_7 \min \left\{ \frac{t^2 C_{ij}^2}{\sigma^4 K^4 p}, \frac{t C_{ij}}{\sigma^2 K^2}\right\}\right).
$$
Then we use that $C_{ij} \geq \sigma^2 p$ and $t \geq M\sqrt{(\log n) /p}$ to get
$$\min \left\{ \frac{t^2 C_{ij}^2}{\sigma^4 K^4 p}, \frac{t C_{ij}}{\sigma^2 K^2}\right\} \geq \frac{M}{K^2}(\log n)\min \left \{\frac{M}{K^2}, \frac{p^{1/2}}{\sqrt{\log n}} \right\}.$$
Finally, using that $p \geq \min_{i \in [n]} p_{\text{int}}^{(i)} \in \Omega (\log n)$, it follows that there exists $c_8 \geq 0$ and $n_3$ such that for $n \geq n_3$, $\sqrt{p/ \log n} \geq c_8$. Therefore, we get that for $n \geq n_3$,
$$
\P\left(\left|\sigma^2 (\mathbf{E}_i\cdot\mathbf{E}_j -\E[\mathbf{E}_i\cdot\mathbf{E}_j]) \right|>\frac{t}{8} C_{ij}\right)\leq 2 \exp \left( - \frac{c_7M}{K^2}(\log n)\min \left \{\frac{M}{K^2}, c_8 \right\}\right).
$$

\qed

\begin{proof}[Proof of theorem \ref{thm:rand_func}]

Define
\[
\pi_{ij}(t) := \mathbb{P}\left( \left|\frac{\mathbf{Y}_i \cdot \mathbf{Y}_j }{\mathbb{E}[\|\mathbf{Y}_i\|^2]^{1/2}\mathbb{E}[\|\mathbf{Y}_j\|^2]^{1/2}}- \frac{\mathbb{E}[\mathbf{Y}_i \cdot \mathbf{Y}_j]}{\mathbb{E}[\|\mathbf{Y}_i\|^2]^{1/2}\mathbb{E}[\|\mathbf{Y}_j\|^2]^{1/2}} \right|  > t\right).
\]
Using our definition of $C_{ij}$, we can rewrite this as
$$\pi_{ij}(t) = \mathbb{P}\left(\left| \mathbf{Y}_i \cdot \mathbf{Y}_j - \mathbb{E}[\mathbf{Y}_i \cdot \mathbf{Y}_j ]\right| > t C_{ij} \right).$$
We can then decompose $\mathbf{Y}_i \cdot \mathbf{Y}_j - \mathbb{E}[\mathbf{Y}_i \cdot \mathbf{Y}_j ]$ in the following way:
\begin{align*}
    &\Y_i\cdot \Y_j  - \E[\Y_i\cdot \Y_j ] \nonumber\\
    &=\X(z_i)^\top\SS_{ij}\X(z_j) - \E[\X(z_i)^\top\SS_{ij}\X(z_j)] \\
    &+\MMu(z_i)^\top \SS(z_j)^{1/2}\X(z_j) + \MMu(z_j)^\top \SS(z_i)^{1/2}\X(z_i)\\
    &+\sigma\mathbf{E}_i^\top \SS(z_j)^{1/2}\X(z_j) + \sigma\mathbf{E}_j^\top \SS(z_i)^{1/2}\X(z_i)\\
    &+\sigma\mathbf{E}_i^\top \MMu(z_j) + \sigma\mathbf{E}_j^\top \MMu(z_i)\\
    &+\sigma^2 (\mathbf{E}_i\cdot\mathbf{E}_j -\E[\mathbf{E}_i\cdot\mathbf{E}_j]) .
\end{align*}
Using this decomposition, and a union bound to produce the result in terms of the maximum difference over all $i, j \in [n]$, we get
\begin{align}
    \pi_{\max} &= \mathbb{P}\left(\max_{i,j \in [n]} \left|\mathbf{Y}_i \cdot \mathbf{Y}_j - \mathbb{E}[\mathbf{Y}_i \cdot \mathbf{Y}_j]\right| > t C_{ij}\right)  \leq \sum_{ij}\pi_{ij}(t) \nonumber \\
    \label{eq:pi_decomp_first}
     &\leq \sum_{ij}\Bigg[ \mathbb{P}\left(\left|\X(z_i)^\top\SS_{ij}\X(z_j) - \E[\X(z_i)^\top\SS_{ij}\X(z_j)] \right| > \frac{t}{8} C_{ij} \right)\\
    &+\mathbb{P} \left( \left|\MMu(z_i)^\top \SS(z_j)^{1/2}\X(z_j) \right| > \frac{t}{8} C_{ij}  \right) +  \mathbb{P} \left( \left|\MMu(z_j)^\top \SS(z_i)^{1/2}\X(z_i) \right| > \frac{t}{8} C_{ij}  \right)\\
    &+ \mathbb{P} \left( \left| \sigma\mathbf{E}_i^\top \SS(z_j)^{1/2}\X(z_j)\right| > \frac{t}{8} C_{ij}  \right)+ \mathbb{P} \left( \left|\sigma\mathbf{E}_j^\top \SS(z_i)^{1/2}\X(z_i)  \right| > \frac{t}{8} C_{ij}  \right)\\
    &+ \mathbb{P} \left( \left|\sigma\mathbf{E}_i^\top \MMu(z_j) \right| > \frac{t}{8} C_{ij}  \right) + \mathbb{P} \left( \left|\sigma\mathbf{E}_j^\top \MMu(z_i)  \right| > \frac{t}{8} C_{ij}  \right)\\
    &+ \mathbb{P} \left( \left|\sigma^2 (\mathbf{E}_i\cdot\mathbf{E}_j -\E[\mathbf{E}_i\cdot\mathbf{E}_j]) \right| > \frac{t}{8} C_{ij}  \right) \Bigg].
    \label{eq:pi_decomp_last}
\end{align}
Next we use lemmas \ref{lem:rand_func_first}-\ref{lem:rand_func_last} to bound the terms in expressions (\ref{eq:pi_decomp_first})-(\ref{eq:pi_decomp_last}), respectively. Letting $c \geq c_1, \ldots, c_8$ and $n_4 \geq n_0, \ldots, n_3$, we get that for $n \geq n_4$:
\begin{align*}
    \pi_{\max} &\leq \sum_{ij} \left[ 8 \exp \left( -  (\log n) \frac{cM}{K^2} \min \left\{ \frac{M}{K^2}, c\right\} \right) + 8 \exp\left(-\frac{c M^2}{K^2} (\log n) \right) \right] \\
    &=  8 \exp \left( -  (\log n) \frac{cM}{K^2} \min \left\{ \frac{M}{K^2}, c\right\} + 2 \log n\right) + 8 \exp\left(-\frac{c M^2}{K^2} (\log n) +2 \log n \right).
\end{align*}
Now fix any $\varepsilon >0$. By choosing $M$ large enough so that $\frac{M}{K^2}>c$, and then increasing $M$ if necessary, we can achieve
$$
\pi_{\max} \leq 8 \exp \left(- \log n \left( \frac{c^2 M}{K^2} - 2 \right) \right) + 8 \exp \left( -\log n \left( \frac{cM^2}{K^2} - 2\right)\right) \leq \varepsilon,
$$
and the result of the theorem follows.
\end{proof}

\begin{proof}[Proof of proposition \ref{prop:rand_func}]
We have 
\[
\mathbb{E}[\mathbf{Y}_{i}\cdot\mathbf{Y}_{j}]=\phi(z_{i})\cdot\phi(z_{j})+p\sigma^{2}\mathbf{I}[i=j],
\]
hence
\begin{align*}
 & \left|\frac{\mathbf{Y}_{i}\cdot\mathbf{Y}_{j}}{p}-\frac{\phi(z_{i})\cdot\phi(z_{j})}{p}-\sigma^{2}\mathbf{I}[i=j]\right|\\
 & =\left|\frac{\mathbf{Y}_{i}\cdot\mathbf{Y}_{j}}{\mathbb{E}[\|\mathbf{Y}_{i}\|^{2}]^{1/2}\mathbb{E}[\|\mathbf{Y}_{j}\|^{2}]^{1/2}}-\frac{\mathbb{E}[\mathbf{Y}_{i}\cdot\mathbf{Y}_{j}]}{\mathbb{E}[\|\mathbf{Y}_{i}\|^{2}]^{1/2}\mathbb{E}[\|\mathbf{Y}_{j}\|^{2}]^{1/2}}\right|\frac{\mathbb{E}[\|\mathbf{Y}_{i}\|^{2}]^{1/2}\mathbb{E}[\|\mathbf{Y}_{j}\|^{2}]^{1/2}}{p}\\
 & \leq\left|\frac{\mathbf{Y}_{i}\cdot\mathbf{Y}_{j}}{\mathbb{E}[\|\mathbf{Y}_{i}\|^{2}]^{1/2}\mathbb{E}[\|\mathbf{Y}_{j}\|^{2}]^{1/2}}-\frac{\mathbb{E}[\mathbf{Y}_{i}\cdot\mathbf{Y}_{j}]}{\mathbb{E}[\|\mathbf{Y}_{i}\|^{2}]^{1/2}\mathbb{E}[\|\mathbf{Y}_{j}\|^{2}]^{1/2}}\right|\frac{1}{p}\max_{i\in[n]}\mathbb{E}[\|\mathbf{Y}_{i}\|^{2}]
\end{align*}
where 
\[
\frac{1}{p}\max_{i\in[n]}\mathbb{E}[\|\mathbf{Y}_{i}\|^{2}]=\max_{i\in[n]}\frac{\mathrm{tr}[\boldsymbol{\Sigma}(z_{i})]+\|\boldsymbol{\mu}(z_{i})\|^2}{p}+\sigma^{2}.
\]
 The first claim of the proposition then follows from theorem \ref{thm:rand_func}.

For the second claim of the proposition, note that for $i\neq j$,
\begin{multline*}
\frac{\CosSim(\phi(z_i),\phi(z_j))}{\gamma_{ij}(\sigma)}\\= \frac{\phi(z_i)\cdot \phi(z_j)}{(\|\phi(z_i)\|^2+p\sigma^2)^{1/2}(\|\phi(z_j)\|^2+p\sigma^2)^{1/2}} = \frac{\E[\Y_i\cdot\Y_j]}{\E[\|\Y_i\|^2]^{1/2}\E[\|\Y_j\|^2]^{1/2}}.
\end{multline*}
So we need to prove that 
\begin{equation}\label{eq:cossim_objective}
\max_{i\neq j\in[n]}\left|\frac{\Y_i\cdot \Y_j}{\|\Y_i\|\|\Y_j\|}-\frac{\E[\Y_i\cdot\Y_j]}{\E[\|\Y_i\|^2]^{1/2}\E[\|\Y_j\|^2]^{1/2}}\right| \in O_{\P}\left(\sqrt{\frac{\log n}{\min_{i\in[n]}p_{\mathrm{int}}^{(i)}}}\right). 
\end{equation}
The only differences between this and the result of theorem \ref{thm:rand_func} is that here we only need to consider the maximum over $i\neq j$, and the normalisation by $\|\Y_i\|\|\Y_j\|$ of the first term within the absolute value in \eqref{eq:cossim_objective}.

Consider the decomposition:
\begin{align}
&\frac{\Y_i\cdot \Y_j}{\|\Y_i\|\|\Y_j\|}-\frac{\E[\Y_i\cdot\Y_j]}{\E[\|\Y_i\|^2]^{1/2}\E[\|\Y_j\|^2]^{1/2}}\nonumber \\
&=\frac{\Y_i\cdot\Y_j}{\E[\|\Y_i\|^2]^{1/2}\E[\|\Y_j\|^2]^{1/2}} - \frac{\E[\Y_i\cdot\Y_j]}{\E[\|\Y_i\|^2]^{1/2}\E[\|\Y_j\|^2]^{1/2}}\label{eq:cossim_decomp1} \\
&+ \Y_i\cdot\Y_j\left[\frac{1}{\|\Y_i\|\|\Y_j\|} -\frac{1}{\E[\|\Y_i\|^2]^{1/2}\E[\|\Y_j\|^2]^{1/2}} \right].\label{eq:cossim_decomp2}
\end{align}
Theorem \ref{thm:rand_func} implies that the maximum over $i\neq j$ of the term in \eqref{eq:cossim_decomp1} is in $O_{\P}\left(\sqrt{\frac{\log n}{\min_{i\in[n]}p_{\mathrm{int}}^{(i)}}}\right)$.

Consider the term in \eqref{eq:cossim_decomp2}. Using the triangle inequality, Cauchy-Schwarz, and triangle inequality again, we have:
\begin{align}
&\left|\Y_i\cdot\Y_j\left[\frac{1}{\|\Y_i\|\|\Y_j\|} -\frac{1}{\E[\|\Y_i\|^2]^{1/2}\E[\|\Y_j\|^2]^{1/2}} \right]\right|\nonumber\\
&= \left| \frac{\Y_i
}{\|\Y_i\|}\cdot\left(\frac{\Y_j
}{\|\Y_j\|}-\frac{\Y_j
}{\E[\|\Y_j\|^2]^{1/2}}\right) + \frac{\Y_j}{\E[\|\Y_j\|^2]^{1/2}}\cdot\left(  \frac{\Y_i}{\|\Y_i\|} - 
 \frac{ \Y_i}{\E[\|\Y_i\|^2]^{1/2}}\right)  \right|\nonumber\\
 &\leq \|\Y_j\|\left|\frac{1}{\|\Y_j\|}-\frac{1}{\E[\|\Y_j\|^2]^{1/2}}\right| + \frac{ \|\Y_j\|}{\E[\|\Y_j\|^2]^{1/2}}\|\Y_i\|\left|\frac{1}{\|\Y_i\|}-\frac{1}{\E[\|\Y_i\|^2]^{1/2}} \right|\nonumber\\
 &\leq \left|\frac{\|\Y_j\|}{\E[\|\Y_j\|^2]^{1/2}}-1\right| + \left|\frac{\|\Y_j\|}{\E[\|\Y_j\|^2]^{1/2}}-1\right|\left|\frac{\|\Y_i\|}{\E[\|\Y_i\|^2]^{1/2}}-1\right| +  \left|\frac{\|\Y_i\|}{\E[\|\Y_i\|^2]^{1/2}}-1\right|.\label{eq:cossim_norm_decomp}
\end{align}
Theorem \ref{thm:rand_func} implies
\begin{equation}\label{eq:norm_bound}
\max_i \left|\frac{\|\Y_i\|^2}{\E[\|\Y_i\|^2]}-1\right|\in O_{\mathbb{P}}\left(\sqrt{\frac{\log n}{ \min_i p_{\mathrm{int}}^{(i)
}}}\right),
\end{equation} i.e., for any $\epsilon>0$ there exists $M$ and $n_0$ such that for any $n\geq n_0$, 
$$
\mathbb{P}\left(\max_i \left|\frac{\|\Y_i\|}{\E[\|\Y_i\|^2]^{1/2}}-1\right|>M\sqrt{\frac{\log n}{\min_i p_{\mathrm{int}}^{(i)
}}}\right)\leq \epsilon.
$$
Now for any $a>0$, $(a+1)|a-1|=|a^2-1|$, i.e. $|a-1|\leq |a^2-1|$, hence 
$$
\max_i \left|\frac{\|\Y_i\|}{\E[\|\Y_i\|^2]^{1/2}}-1\right|>M\sqrt{\frac{\log n}{\min_i p_{\mathrm{int}}^{(i)
}}} \quad \Rightarrow \quad\max_i \left|\frac{\|\Y_i\|^2}{\E[\|\Y_i\|^2]}-1\right|>M\sqrt{\frac{\log n}{\min_i p_{\mathrm{int}}^{(i)
}}},
$$  
so \eqref{eq:norm_bound} implies:
$$
\max_i \left|\frac{\|\Y_i\|}{\E[\|\Y_i\|^2]^{1/2}}-1\right|\in O_{\mathbb{P}}\left(\sqrt{\frac{\log n}{ \min_i p_{\mathrm{int}}^{(i)
}}}\right).
$$
Thus, via \eqref{eq:cossim_norm_decomp}, the maximum over $i\neq j$ of the term in \eqref{eq:cossim_decomp2} is in $O_{\mathbb{P}}\left(\sqrt{\frac{\log n}{ \min_i p_{\mathrm{int}}^{(i)
}}}\right)$.
We have thus shown that \eqref{eq:cossim_objective} holds as required, and that completes the proof of the second claim of the proposition.
\end{proof}

\begin{proof}[Proof of lemma \ref{lem:homeo}]
We have:
\begin{align*}
\E[\|\Y^{\mathrm{nf}}(z)-\Y^{\mathrm{nf}}(z^{\prime})\|^2] &= \E[\|\Y^{\mathrm{nf}}(z)\|^2] + \E[\|\Y^{\mathrm{nf}}(z^{\prime})\|^2]  -2 \E[\Y^{\mathrm{nf}}(z)\cdot \Y^{\mathrm{nf}}(z^{\prime})]\\
 &= \|\phi(z)\|^2 + \|\phi(z^{\prime}\|^2  - 2 \phi(z)\cdot \phi(z^{\prime})\\
 &= \|\phi(z) - \phi(z^{\prime})\|^2.
\end{align*}
Therefore assumption \textbf{A\ref{ass:MScontinuity}} is equivalent to continuity of $\phi$, and \textbf{A\ref{ass:homeo}} implies $\phi$ is invertible on its image $\*M$. A result in the theory of metric spaces \citeapp[Prop. 13.26]{sutherland2009introduction}, states that any continuous invertible mapping on a compact domain has an continuous inverse. Therefore the inverse of $\phi$ is continuous. 
\end{proof}

\section{Supporting material for section \ref{sec:tda_consistency}}
\subsection{Bounding the bottleneck distance}
First let us collect the following definitions and results from \cite[Section 2.1]{chazal2013optimal}:
\begin{enumerate}
\item The Hausdorff distance between two compact subsets $C_1$ and $C_2$ of a metric space is
\[d_{\text{H}}(C_1, C_2) = \max\{\sup_{x\in C_1}d(x,C_2), \sup_{x \in C_2} d(x,C_1)\},\]
where the distance $d(x, C)$ to a compact subset $C$ is the minimum over all $y \in C$ of $d(x,y)$.
\item Two compact metric spaces are said to be isometric if there exists a bijection between them that preserves distances.
\item The Gromov-Hausdorff distance $d_{\text{GH}}(\*M,\*M^{\prime})$ between two compact metric spaces $\*M$ and $\*M^{\prime}$ is the infimum of the real numbers $r \geq 0$ such that there exists a metric space with compact subspaces $C_1$ and $C_2$ which are isometric to $\*M$ and $\*M^{\prime}$ respectively and such that $d_{\text{H}}(C_1,C_2) < r$. If $\*M$ and $\*M^{\prime}$ are compact subsets of the same metric space, then $d_{\text{GH}}(\*M,\*M^{\prime}) \leq d_{\text{H}}(\*M,\*M^{\prime})$.
\item Let $\text{Filt}(\*M)$ denote either the Rips, \v{C}ech, or Alpha filtration of $\*M$, $\text{dgm}(\text{Filt}(\*M))$ the associated persistent diagram (simply writing $\text{dgm}(\*M)$ for brevity in the main document), and $d_{\text{b}}$ the bottleneck distance between persistence diagrams. Then, for two compact metric spaces $\*M, \*M^{\prime}$, 
  \[d_{\text{b}}\left(\text{dgm}(\text{Filt}(\*M)), \text{dgm}(\text{Filt}(\*M^{\prime}))\right) \leq 2 d_{\text{GH}}(\*M, \*M^{\prime}).\]
\end{enumerate}
Applying the above,
\begin{align*}
  d_{\text{b}}\left(\text{dgm}(\text{Filt}(\*{Y}_n)), \text{dgm}(\text{Filt}(\*M))\right) & \leq 2 d_{\text{GH}}(\*{Y}_n, \*M) \leq 2 (d_{\text{GH}}(\*{Y}_n,  \*{M}_n) + d_{\text{GH}}(\*{M}_n, \*M))\\ &\leq 2 (d_{\text{GH}}(\*{Y}_n,  \*{M}_n) + d_{\text{H}}(\*{M}_n, \*M))
\end{align*}

\subsection{Bound on the Gromov-Hausdorff distance}\label{sec:GH_distance}
Here, we prove that
\[d_{\text{GH}}^2(p^{-1/2}\mathcal{Y}_n,  p^{-1/2}\mathcal{M}_n) \leq \max_{i,j \in [n]}\frac{1}{p}\left|\Y_i \cdot \Y_j - \phi(z_i) \cdot \phi(z_j) \right|.\]
Let $\+W_i \coloneqq \phi(z_i)$. Then, 
\begin{align*}
  &\left| \lVert \+Y_i - \+Y_j \rVert^2 - \lVert \+W_i - \+W_j \rVert^2\right|\\ &= \left| (\+Y_i \cdot \+Y_i - \+W_i \cdot \+W_i) + (\+Y_j \cdot \+Y_j - \+W_j \cdot \+W_j) - 2(\+Y_i \cdot \+Y_j - \+W_i \cdot \+W_j)\right|\\
  &\leq 4 \epsilon.
\end{align*}
where $\epsilon \coloneqq \max_{i,j \in [n]} \left|\+Y_i \cdot \+Y_j - \+W_i \cdot \+W_j \right|.$

Applying e.g. \cite[Lemma 6]{whiteley2021matrix}, we have $\left||a|-|b|
\right| \leq |a^2- b^2|^{1/2}$ for any $a,b \in \R$, therefore,
\begin{equation}\label{eq:diff_of_dist}
\left| \lVert \+Y_i - \+Y_j \rVert - \lVert \+W_i - \+W_j \rVert \right| \leq \sqrt{4 \epsilon},\end{equation}
for all $i,j \in [n]$.

Now, let us recall some definitions and results from \cite{ivanov2016realizations}:
\begin{enumerate}
\item A \emph{correspondence} between two metric spaces $(\*M, d)$ and $(\*M^\prime,d^\prime)$ is a subset $R \subset \*M \times \*M^\prime$ such that the canonical projections $\pi_1 : (x,y) \mapsto x$ and $\pi_1 : (x,y) \mapsto y$ for $(x,y) \in R$ are surjective, and the set of all correspondences is denoted $\*R(\*M, \*M^\prime)$.
\item The distortion of a correspondence is
  \[\text{dis}(R) \coloneqq \sup\left\{\left|d(x,x') - d^\prime(y,y')\right| : (x,y), (x',y') \in R\right\},\]
\item The Gromov-Hausdorff distance $d_{\text{GH}}(\*M,\*M^\prime)$ quantifies the `best correspondence' \citep[Theorem 1.1]{ivanov2016realizations},
\[d_{\text{GH}}(\*M,\*M^\prime) = \frac{1}{2} \inf \{ \text{dis}(R) : R \in \*R(\*M, \*M^\prime)\}.\]
\end{enumerate}
The set $R_0 = \{(\+Y_1,\+W_1), \ldots, (\+Y_n,\+W_n)\}$ is a correspondence between $\*Y_n$ and $\*M_n$, therefore,
\[d_{\text{GH}}(\*Y_n,\*M_n) = \frac{1}{2} \inf \{ \text{dis}(R) : R \in \*R(\*Y_n, \*M_n) \leq \frac{1}{2}\text{dis}(R_0) \leq \frac{1}{2}\sqrt{4 \epsilon} = \sqrt{\epsilon},\]where the penultimate inequality uses \eqref{eq:diff_of_dist}.
Therefore, 
\begin{multline*}
d^2_{\text{GH}}(p^{-1/2}\*Y_n,p^{-1/2}\*M_n) = \frac{1}{p} d^2_{\text{GH}}(\*Y_n,\*M_n) \leq \frac{1}{p} \epsilon \\  = \max_{i,j \in [n]}\frac{1}{p}\left|\Y_i \cdot \Y_j - \phi(z_i) \cdot \phi(z_j) \right|\leq \max_{i,j \in [n]}\frac{1}{p}\left|\Y_i \cdot \Y_j - \phi(z_i) \cdot \phi(z_j) -\sigma^2 \mathbf{I}[i=j]\right|+\sigma^2\\
,
\end{multline*}
as claimed at the end of section \ref{sec:tda_consistency}.

\section{TDA and manifold learning using normalised data \texorpdfstring{$\Y_i/\|\Y_i\|$}{}}
In connection with the second part of proposition \ref{prop:rand_func}, we can consider TDA and manifold learning  based on the `self-normalised' data,
\begin{align*}\overline{\*Y}_n&\coloneqq\{\Y_i/\|\Y_i\|;i\in[n]\}\\
\overline{\*M}_n &\coloneqq\{\phi(z_i)/\|\phi(z_i)\|;i\in[n]\},\\ \overline{\*M}&\coloneqq \{\phi(z)/\|\phi(z)\|;z\in\*Z\}.
\end{align*}
A practical advantage of this `self-normalised' setting is that we do not need to decide on a specific rescaling factor such as $p$ in the first part of proposition \ref{prop:rand_func}.

We can consider the question of how $\overline{\*M}$ relates to $\*Z$, analogously to lemma \ref{lem:homeo} but under the following assumption instead of \textbf{A\ref{ass:homeo}}. 
\begin{assumption}\label{ass:homeo_self_norm}
If $z\neq z^\prime$, then:
$$
\E\left[\left\lVert\frac{\Y^{\mathrm{nf}}(z)}{\E[\|\Y^{\mathrm{nf}}(z)\|^2]^{1/2}}-\frac{\Y^{\mathrm{nf}}(z^\prime)}{\E[\|\Y^{\mathrm{nf}}(z^\prime)\|^2]^{1/2}}\right\rVert^2\right]>0.$$
\end{assumption}
\begin{lem}\label{lem:homeo_self_norm}
Under \textbf{A\ref{ass:MScontinuity}}, \textbf{A\ref{ass:Z_compact}} and \textbf{A\ref{ass:homeo_self_norm}}, $z\mapsto\phi(z)/\|\phi(z)\|$ is a homeomorphism between $\*Z$ and $\overline{\*M}$.
\end{lem}
\begin{proof}
Using \eqref{eq:mercer},
\begin{align*}
\frac{1}{2}\left\lVert\frac{\phi(z)}{\|\phi(z)\|}-\frac{\phi(z^\prime)}{\|\phi(z^\prime)\|}\right\rVert^2 & =  1-\frac{\phi(z)\cdot\phi(z^\prime)}{\|\phi(z)\|\|\phi(z^\prime)\|} \\
& = 1-\frac{\E[\Y^{\mathrm{nf}}(z)\cdot \Y^{\mathrm{nf}}(z^\prime)]}{\E[\|\Y^{\mathrm{nf}}(z)\|^2]^{1/2}\E[\|\Y^{\mathrm{nf}}(z^\prime)\|^2]^{1/2}} \\
& = \frac{1}{2}\E\left[\left\lVert\frac{\Y^{\mathrm{nf}}(z)}{\E[\|\Y^{\mathrm{nf}}(z)\|^2]^{1/2}}-\frac{\Y^{\mathrm{nf}}(z^\prime)}{\E[\|\Y^{\mathrm{nf}}(z^\prime)\|^2]^{1/2}}\right\rVert^2 \right].
\end{align*}
The claim of the lemma then follows by the same arguments as in the proof lemma \ref{lem:homeo}.
\end{proof}
Lemma \ref{lem:homeo_self_norm} motivates the application of TDA to $\overline{\*Y}_n$ to discover the structure of $\*Z$. Following the same line of investigation as in section \ref{sec:GH_distance}, in order to bound $d^2_{\text{GH}}(\overline{\*Y}_n,\overline{\*M}_n)$ we can follow similar arguments as above subject to some small changes: let $\+W_i \coloneqq \phi(z_i)/\|\phi(z_i)\|$, so $\|\+W_i\|=1$, and 
\begin{align*}
  &\left| \left\lVert \frac{\+Y_i}{\|\Y_i\|} - \frac{\+Y_j}{\|\Y_j\|} \right\rVert^2 - \lVert \+W_i - \+W_j \rVert^2\right|\\ &= \left| \left(\frac{\+Y_i}{\|\Y_i\|} \cdot \frac{\+Y_i}{\|\Y_i\|} - \+W_i \cdot \+W_i\right) + \left(\frac{\+Y_j}{\|\Y_j\|} \cdot \frac{\+Y_j}{\|\Y_j\|} - \+W_j \cdot \+W_j\right) - 2\left(\frac{\+Y_i}{\|\Y_i\|} \cdot \frac{\+Y_j}{\|\Y_j\|} - \+W_i \cdot \+W_j\right)\right|\\
  & = 0 + 0 +2\left|\frac{\+Y_i}{\|\Y_i\|} \cdot \frac{\+Y_j}{\|\Y_j\|} - \+W_i \cdot \+W_j \right|
  \\
  &\leq 2 \epsilon,
\end{align*}
where $\epsilon \coloneqq \max_{i\neq j \in [n]} \left|\frac{\+Y_i}{\|\Y_i\|} \cdot \frac{\+Y_j}{\|\Y_j\|} - \+W_i \cdot \+W_j \right|.$
Following through the same steps as in section \ref{sec:GH_distance} then gives:
$$
d^2_{\text{GH}}(\overline{\*Y}_n,\overline{\*M}_n) \leq \max_{i\neq j}\left|\frac{\Y_i\cdot\Y_j}{\|\Y_i\|\|\Y_j\|} - \frac{\phi(z_i)\cdot \phi(z_j)}{\|\phi(z_i)\|\|\phi(z_j)\|}\right|.
$$
Applying the triangle inequality and adopting the notation of the second part of proposition \ref{prop:rand_func} we therefore have:
$$
d^2_{\text{GH}}(\overline{\*Y}_n,\overline{\*M}_n) \leq \max_{i\neq j}\left|\CosSim(\Y_i,\Y_j)-\frac{\CosSim(\phi(z_i),\phi(z_j))}{\gamma_{ij}(\sigma)}\right| + 1-\frac{1}{\max_{i\neq j}\gamma_{ij}(\sigma)},
$$
where Cauchy-Schwarz and $\gamma_{ij}(\sigma)\geq 1$ have been used. The first $\max_{i\neq j}$ term on the l.h.s. of the above inequality is controlled by proposition \ref{prop:rand_func}, where as the second term goes to zero as $\sigma\to 0$.

The task of looking for evidence of isometry of the mapping $z\mapsto \phi(z)/\|\phi(z)\|$ can be conducted as in section \ref{sec:isometry}, simply replacing $\Y_i$ there by $\Y_i/\|\Y_i\|$; thus comparing shortest path-lengths in $\overline{\*Y}_n$ to those in $\*Z_n$.





\section{Further discussion of the toy example}
Recall in the toy example:
$$
\phi(z)=p^{1/2}\left[z_1,\frac{2}{\pi}\sin(\pi z_2/2),\frac{2}{\pi}\cos(\pi z_2/2)\right],
$$
with $z=(z_1,z_2)\in\{(z_1,z_2):z_1^2 +z_2^2\}\subset \mathbb{R}^2$.
We have:
\begin{align*}
\frac{1}{p}\phi(z)\cdot\phi(z^\prime) &= z_1 z_1^\prime +\frac{4}{\pi^2}\sin(z_2 \pi/2)\sin(z_2^\prime \pi/2) + \frac{4}{\pi^2}\cos(z_2 \pi/2)\cos(z_2^\prime \pi/2).\\
& = z_1 z_1^\prime +\frac{4}{\pi^2} \cos[(z_2-z_2^\prime)\pi/2].
\end{align*}
It can be shown that isometry holds, indeed by direct calculation:
$$
\left.\frac{\partial^2 }{\partial z_i \partial z_j^\prime} \phi(z)\cdot\phi(z^\prime)\right|_{z=z^\prime} = \begin{cases}
p,\quad i=j, \\
0,\quad i\neq j,
\end{cases}
$$
which according to \citep[Thm. 1]{whiteley2021matrix} is a sufficient condition for isometry (the above expression for the partial derivatives means the matrix denoted by $\mathbf{H}_{\eta_t}$ in \citep[Thm. 1]{whiteley2021matrix} is proportional to the identity matrix for all $t$).

The visualisation in figure \ref{fig:sim_Example} was produced as follows. Write $\Y \coloneqq [\Y_1|\cdots|\Y_n]^\top$ and consider the SVD:
$$
\Y = \mathbf{U}_3\mathbf{S}_3\mathbf{V}_3^\top + \mathbf{U}_\perp\mathbf{S}_\perp\mathbf{V}_\perp^\top,
$$
where $\mathbf{S}_3$ is the 3x3 diagonal matrix with diagonal elements which are the three largest singular values, and the columns of $\mathbf{U}_3\in\mathbb{R}^{n\times 3}$, $\mathbf{V}_3\in\mathbb{R}^{p\times 3}$ are associated left/right singular vectors.

Subplots (a), (b), (c) in figure \ref{fig:sim_Example} show for respective values of $p$, the rows of $\mathbf{U}_3\mathbf{S}_3$ as a point-cloud in $\mathbb{R}^3$, rotated to align with (d) (this alignment accounts for the fact that singular vectors are only mathematically defined up to sign, so the sign obtained when computing the SVD in practice is arbitrary).

\section{Details of numerical results}

In this section, we will expand on the details necessary to reproduce the results in section \ref{sec:neuro}.
\citet{gardner2022toroidal} collected data over multiple days, from three different rats, covering various modules (groups of grid cells) within each rat's brain. In all of our experiments, we focus on a single dataset: \{rat `R', module 2, day 1\}, although this choice is arbitrary.

\subsection{Context}

In order to visualise the toroidal structure in the data, as in figure \ref{fig:torus}(a), we follow the approach of \cite{gardner2022toroidal}: perform PCA on the data into six dimensions, followed by UMAP \cite{mcinnes2018umap} into three dimensions, and colour points by the first PC. We follow this procedure for $n=15,000$ points, which were selected by retaining the 15,000 most active points, as measured by the mean firing rate in each time bin. The hyperparameters used for UMAP were: `n\_components'=3, `n\_neighbors'=2000, `min\_dist'=0.8, `metric'=`cosine' and `init'=`spectral'. Running the UMAP algorithm on 15,000 points had the longest running time of any of our experiments, taking 3 minutes 32 seconds. This was run on a Linux-based Dell Latitude laptop equipped with an 11th Gen Intel Core i7-1185G7 CPU (4 cores, 8 threads) and 15GiB of RAM, using Python v3.11. 

We also replicate the persistent homology analysis in \cite{gardner2022toroidal}, using the \texttt{ripser} Python package. Since computing the persistent homology of a set of points is computationally expensive and sensitive to outliers, it is common to downsample and dimension-reduce the data beforehand. Following the approach used for the torus visualisation, we begin by selecting the same 15,000 data points and apply PCA into six dimensions. The choice of six dimensions is explored in detail in \cite{gardner2022toroidal}, where they show that the first 6 dimensions retain a large proportion of the variance in the data. Using the authors' code \cite{gardner_code}, we then downsample 800 points from the PCA embedding using a density-based method informed by the fuzzy topological representation used internally by UMAP. Finally, a distance matrix is produced from this reduced point cloud, which can be passed to \texttt{ripser}. Further details of this procedure can be found in \cite{gardner2022toroidal}. The output of \texttt{ripser} can easily be plotted as a persistence diagram, as in figure \ref{fig:torus}(c). We note that for purposes of visualisation, it was necessary to slightly shift one of the points in persistence diagram corresponding to on of the 1-dimensional holes in order to visually distinguish it from the other, as their birth and death times were close to identical. We also note that \cite{gardner2022toroidal} reported bar-code diagrams, rather than persistence diagrams, but these convey the same information; we show a persistence diagram to connect more directly with section \ref{sec:tda_consistency}.

To produce plots (b) and (d) in figure \ref{fig:torus}, we applied the cohomological decoding procedure introduced by \citeapp{de2009persistent} and implemented in \cite{gardner_code}.
This method produces two circular coordinates, each corresponding to one of the two 1-dimensional holes identified in the persistent homology analysis. These coordinates, referred to as \textit{toroidal coordinates} are plotted in \cite{gardner2022toroidal} as functions of physical space to reveal the underlying periodic structure in the data. \citet{gardner2022toroidal} show that the periodic structure can be represented by two vectors, $\mathbf{r}_1$ and $\mathbf{r}_2$, that together define a rhombus that, when tessellated, captures the repeated pattern in the toroidal coordinates. In (d), we simply plot the tessellated rhombus atop the physical locations (coloured by the first PC). In (b), the points inside each of the rhombi are identified and then superimposed onto the central rhombus. To help the reader visualise how this 2D rhombus relates to a 3D torus, we have added markings in red to illustrate that a torus is formed when opposite edges of the rhombus are ``glued" together. 

\subsection{Our contribution}

In our experiments that build upon \citet{gardner2022toroidal}'s analysis, we work with the normalised data vectors $\tilde{\mathbf{Y}}_i := \mathbf{Y}_i / \|\mathbf{Y}_i\|, i \in [n]$, where the $n=15,000$ points are the same as those in our contextual experiments. From a theoretical perspective, this normalisation acts as a natural rescaling of the data to be applied by default as in the second part of proposition \ref{prop:rand_func};  the first part of proposition  \ref{prop:rand_func} indicates that an alternative scaling, $\Y_i/p^{1/2}$, would make sense if $p^{-1}\max_{i\in[n]} (\mathrm{tr}[\SS(z_i)] +\|\boldsymbol{\mu}(z_i)\|^2)+\sigma^2 \in O(1)$, but in practice we may not be sure that condition is satisfied.

As discussed in section \ref{sec:isometry}, we use a $k$-nearest neighbour graph (in our case $k=10$), with edge weights $W_{ij} = \|\tilde{\mathbf{Y}}_i - \tilde{\mathbf{Y}}_j \|$, to approximate shortest path lengths in the point cloud $\*Y_n := \{\tilde{\mathbf{Y}}_1, \ldots, \tilde{\mathbf{Y}}_n \}$. The shortest paths between points, denoted $\hat{L}(\mathbf{Y}_i, \mathbf{Y}_j)$, are calculated using Dijkstra's algorithm \citeapp{dijkstra1959note} on the constructed graph. To aid visualisation of the pattern in the data over physical space, as illustrated in figure \ref{fig:path_analysis}(a), we smooth the calculated on-manifold path lengths as a weighted combination of the $k=10$ nearest neighbours in physical space. Specifically, for a source position $\xi_i$ and some other physical position $\xi_j$, with neighbourhood $\mathcal{N}_j$, we weight each point $\xi_k \in \mathcal{N}_j$ by 
$$w_k = \frac{\max_{l \in \mathcal{N}_j}(\|\xi_l - \xi_i\| - \|\xi_k - \xi_i\|)}{\sum_{k \in \mathcal{N}_j} \max_{l \in \mathcal{N}_j}(\|\xi_l - \xi_i\| - \|\xi_k - \xi_i\|)},$$
and set the smoothed path length to the source to be
$\tilde{L}(\mathbf{Y}_i, \mathbf{Y}_j) := \sum_{k \in \mathcal{N}_j} w_k \cdot \hat{L}((\mathbf{Y}_i, \mathbf{Y}_k)$. These smoothed path lengths are also used in figures (b) and (c).

The shortest path lengths in the open field, i.e. under Model 1, are calculated similarly: using a $k=10$ nearest neighbour graph constructed from the physical positions $z_i = \xi_i$, followed by Dijkstra's algorithm. For Model 2, the procedure is the same, except the positions $z_i$ are now superimposed onto the central rhombus. Model 3 introduces additional complexity, as paths are allowed to teleport at the edges to the corresponding point on the opposite side. We handle this by `re-tessellating' the superimposed $z_i$ from Model 2 into the 8 surrounding rhombi, yielding a total of $135,000 = 9\times15,000$ points. Each $z_i$ now has 9 corresponding locations: one in each of the rhombi. Keeping the source point fixed in the central rhombus, we compute the shortest path lengths to each of the 9 destination points. These paths are illustrated in figure \ref{fig:re-tessellation}. The minimum of these 9 distances is taken to be the path length under Model 3; in this case, the shortest path is to the point in the rhombus directly above the central rhombus.

\begin{figure}
    \centering
    \includegraphics[width=0.6\linewidth]{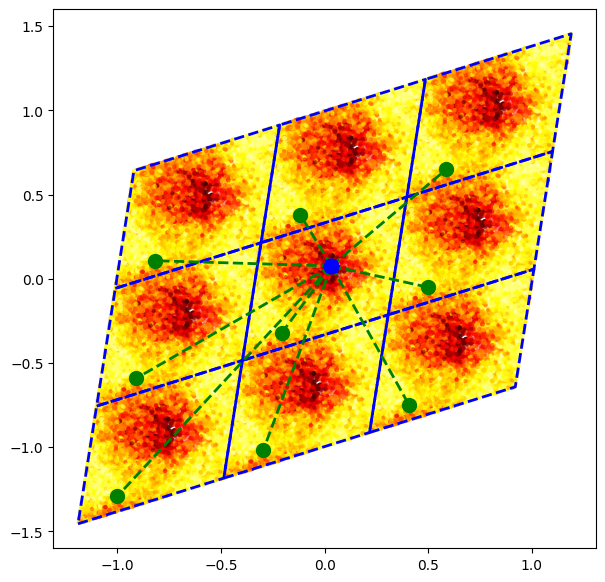}
    \caption{Shortest path lengths on $\*Z_n$ from Model 3 are calculated by first re-tessellating the $z_i$ from Model 2 into the 8 rhombi surrounding the central rhombus (illustrated above). Shortest paths to the same destination point in each of the rhombi are calculated and the minimum of the 9 distances is taken to be the distance under Model 3. Here, the shortest path is to the point in the rhombus directly above the central rhombus.}
    \label{fig:re-tessellation}
\end{figure}

On examination of the three distance-distance plots in figure \ref{fig:path_analysis}(c), we see distinct relationships between distances on the manifold $\*Y_n$ and distances under models 1-3. Under Model 1, the relationship appears approximately linear for points near the source (within a distance of 0.4 on $\*Z_n$). Beyond this threshold, however, distances on $\*Y_n$  begin to decrease as distances on $\*Z_n$ increase, with the trend reversing again around a distance of 0.8 on $\*Z_n$. This relationship can be intuitively explained by the toroidal structure in $\*Y_n$: as the rat moves in a straight line through physical space, its corresponding position on $\*Y_n$ wraps around the torus, eventually approaching its starting position from the opposite direction. 

Under Model 2, superimposing all points onto the central rhombus essentially maps the entire physical space onto a single, flat torus (a torus whose surface has zero Gaussian curvature everywhere). This accounts for the repeating spacial pattern observed in the data. However, flattening the torus into a 2D shape introduces two `cuts' - effectively ``slicing'' the torus - which restricts the available paths. As a result, some shortest paths on the torus become inaccessible after slicing, forcing detours that increase the measured distance. This constraint explains the deviation from a linear trend for distances greater than around 0.4 on $\mathcal{Z}_n$. Model 3 addresses this limitation by employing our `re-tessellation' procedure, which allows paths to traverse the full toroidal structure. While a small cluster of points still falls below the linear trend, we suspect this is due to some error in estimating the vectors that define the rhomboidal flat torus. 

Finally, note that in all three plots we have a positive intercept. As discussed in section \ref{sec:isometry}, assuming our random function model for the data, this offset is due to the positive noise term $\sigma$ and its appearance in equation \eqref{eq:distance}, recalled here:
\begin{equation*}
\frac{1}{p}\|\Y_i-\Y_j\|^2\approx \frac{1}{p}\|\phi(z_i)-\phi(z_j)\|^2 +2\sigma^2.
\end{equation*}
The moving averages, plotted in red, have been computed using a window size of $0.01 n = 150$, with the shaded region showing $\pm 1$ standard deviation.
\begin{figure}[h!]
    \centering
    \includegraphics[width=1\linewidth]{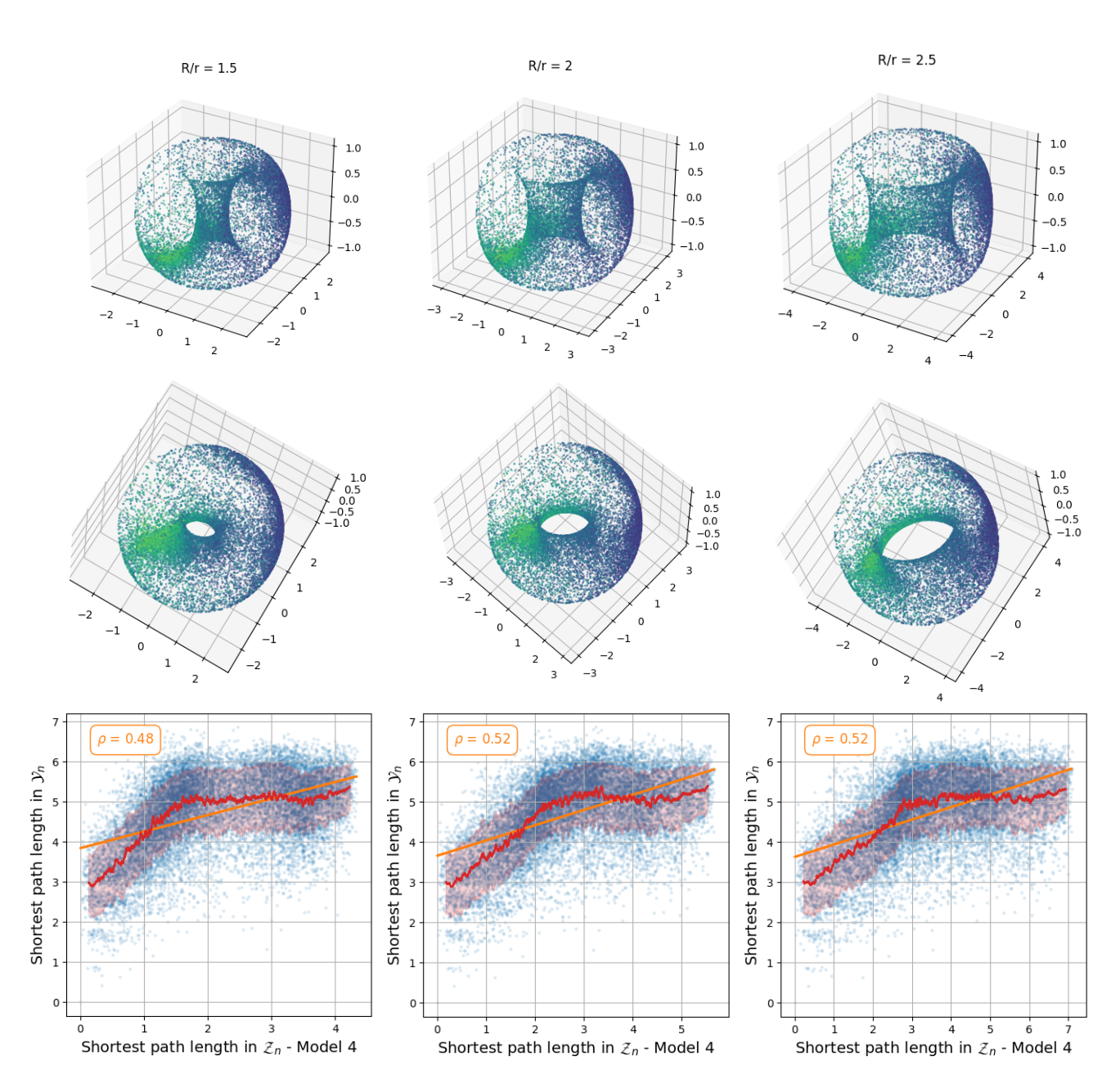}
    \caption{The top two rows show $z_i$ under Model 4 for three different ratios of torus radii $R/r \in \{1.5, 2, 2.5\}$ (left to right). The bottom row shows distance-distance plots comparing the shortest paths on $\*Y_n$ and $\*Z_n$ for each torus. Unlike Model 3, these embeddings exhibit a pronounced deviation from linearity, especially at larger distances. }
    \label{fig:model_4}
\end{figure}

\section{Additional experiments}

Building on the three models discussed in section \ref{sec:neuro}, we introduce a fourth: in `Model 4', the $z_i$ are points on a torus embedded in $\mathbb{R}^3$. We construct this model by taking the decoded toroidal coordinates from \cite{gardner2022toroidal} and mapping them onto a standard 3D torus. The geometry of a 3D torus is not only governed by its circular coordinates, but also by the ratio of the two radii: the larger radius $R$, which defines the circular path around around the central cavity, and the smaller radius $r$, which defines the cross-section of the tube. 

The top two rows of figure \ref{fig:model_4} show the $z_i$ under Model 4 for three different values of the ratio $R/r \in \{1.5, 2, 2.5\}$. The bottom row plots the shortest path lengths on $\*Y_n$  against the shortest path lengths on $\*Z_n$, which were calculated in the same way as for Models 1-3. We see that the linear fit under Model 4 is much worse than Model 3 in the main text.

The fact that there exists no $C^2$ (twice continuously differentiable) isometric embedding of a flat torus into $\mathbb{R}^3$ (proved by \citeapp{hartman1959spherical}, for example) means that shortest paths (and hence shortest path-lengths) in $\*Z$ under Model 3 are not equal to those under Model 4. Assuming that $\*Z$ from Model 3, i.e. the rhomboidal flat torus, is indeed isometric to $\*M$, the non-isometry between $\*Z$ under Models 3 and 4 explains
the non-isometric relationship between distances on $\*M$ and $\*Z$ from Model 4, observed in the bottom row of figure \ref{fig:model_4}.
Intuitively, while distances may be preserved on the initial ``roll up" of the rhombus into a tube, the final step — joining the tube's ends to form a torus — must involve stretching/compressing, depending on the radii of the torus. Suppose that the tube is only stretched (and not compressed) to form the torus; this distortion would inflate distances on $\*Z$ from Model 4 (compared to Model 3), providing an explanation for the observed deviation from linearity in the bottom row plots of figure \ref{fig:model_4}.

\bibliographystyleapp{plainnat}  
\bibliographyapp{refs_appendix}

\end{document}